\documentclass{article}

\usepackage{microtype}
\usepackage{graphicx}
\usepackage{subcaption}
\usepackage{booktabs} 
\usepackage{arydshln} 
\usepackage{float}    
\usepackage{placeins} 
\usepackage{etoc}     

\usepackage{hyperref}

\makeatletter
\let\icml@origaddcontentsline\addcontentsline
\makeatother

\usepackage[preprint]{icml2026}

\usepackage{amsmath}
\usepackage{amssymb}
\usepackage{mathtools}
\usepackage{amsthm}

\usepackage[capitalize,noabbrev]{cleveref}

\theoremstyle{plain}
\newtheorem{theorem}{Theorem}[section]
\newtheorem{proposition}[theorem]{Proposition}
\newtheorem{lemma}[theorem]{Lemma}
\newtheorem{corollary}[theorem]{Corollary}
\theoremstyle{definition}

\newtheorem{assumption}[theorem]{Assumption}
\theoremstyle{remark}
\newtheorem{remark}[theorem]{Remark}

\usepackage[disable,textsize=tiny]{todonotes}

\icmltitlerunning{}

\begin{document}

\twocolumn[
  \icmltitle{Nearly Optimal Bayesian Inference for Structural Missingness}



  \icmlsetsymbol{equal}{*}
  
  \begin{icmlauthorlist}
    \icmlauthor{Chen Liang}{hit}
    \icmlauthor{Donghua Yang}{hit}
    \icmlauthor{Yutong Zhao}{hit}
    \icmlauthor{Tianle Zhang}{hit}
    \icmlauthor{Shenghang Zhou}{hit}
    \icmlauthor{Zhiyu Liang}{hit}
    \icmlauthor{Hengtong Zhang}{hit}
    \icmlauthor{Hongzhi Wang}{hit}
    \icmlauthor{Ziqi Li}{hit}
    \icmlauthor{Xiyang Zhang}{hit}
    \icmlauthor{Zheng Liang}{hit}
    \icmlauthor{Yifei Li}{hit}
  \end{icmlauthorlist}

  \icmlaffiliation{hit}{Harbin Institute of Technology, Harbin, China}

  \icmlcorrespondingauthor{Hongzhi Wang}{wangzh@hit.edu.cn}

  \icmlcorrespondingauthor{Donghua Yang}{yang.dh@hit.edu.cn}

  \icmlkeywords{General Machine Learning, Causality}

  \vskip 0.3in
]



\printAffiliationsAndNotice{}  
\begin{abstract}
  Structural missingness breaks 'just impute and train': values can be undefined by causal or logical constraints, and the mask may depend on observed variables, unobserved variables (MNAR), and other missingness indicators. It simultaneously brings (i) a \emph{catch-22 situation with causal loop}, prediction needs the missing features, yet inferring them depends on the missingness mechanism, (ii) under MNAR, \emph{the unseen are different}, the missing part can come from a shifted distribution, and (iii) plug-in imputation, a single fill-in can lock in uncertainty and yield overconfident, biased decisions. In the Bayesian view, prediction via the posterior predictive distribution integrates over the full model posterior uncertainty, rather than relying on a single point estimate. This framework decouples (i) learning an \emph{in-model} missing-value posterior from (ii) label prediction by optimizing the predictive posterior distribution, enabling posterior integration. This decoupling yields an in-model \emph{almost-free-lunch}: once the posterior is learned, prediction is plug-and-play while preserving uncertainty propagation. It achieves SOTA on 43 classification with missing data and 15 imputation benchmarks, with finite-sample near Bayes-optimality guarantees under our SCM prior.
\end{abstract}

\section{Introduction}
\label{submission}

\textbf{Structural Missingness.} In many real-world applications, data is often incomplete due to various reasons such as sensor failures, terms of data collection, or privacy concerns. Rubin's theory of missing data~\cite{rubin1976inference} proposed a framework for handling missing data, where the missing data is assumed to be (1) missing completely at random (MCAR), where the missingness is independent of the observed and unobserved variables, (2) missing at random (MAR), where the missingness is dependent on the observed variables, or (3) missing not at random (MNAR), where the missingness is dependent on the unobserved variables.

Since MCAR/MAR/MNAR classify missingness by whether $\mathbf{M}$ depends on observed/unobserved entries of $X$, we focus on \emph{structural missingness}: some values are \emph{logically undefined} and masks can exhibit within-mask dependencies (edges $M\!\to\!M$)~\cite{mitra2023learning}. In our scope, we model this via a second-order SCM with randomized missingness mechanisms but a fixed Markov structure, in Sec.~\ref{sec:posterior_distribution}.

Though previous missingness types provides a framework for understanding missing data, reaching a joint solution for fitting the properties of the above missingness types without bias has many challenges as follows:

\textbf{Challenge 1: Alleviating MNAR bias is still an open problem.} Bias may be caused \emph{only under MNAR} by adopting the assumption that observed and unobserved variables have identical distributions, and this issue does not arise under MCAR/MAR~\cite{muzellec2020missing}. Under MNAR, the missingness mechanism depends on unobserved values, so the conditional distribution of missing features given observations can be shifted relative to the observed part and is not identifiable without structural assumptions. We make these assumptions explicit via our second-order SCM prior (Sec.~\ref{sec:posterior_distribution}, in Remark~\ref{rem:mnar-ident})~\cite{muzellec2020missing, ot4sl, zhang2025diffputer}. We summarize this as: $P(x\mid M=1)$, values that end up missing, can differ from $P(x\mid M=0)$, values that are observed. Consequently, standard practice, i.e., training and predicting as if missingness were MAR/ignorable, or filling in a single point estimate (often implicitly treating missing values as drawn from $P(x\mid M=0)$), can induce systematic prediction bias, which is named \textbf{MNAR bias}.

\textbf{Challenge 2: Obstacles in Explicit Modeling for Prior Distribution.} Even with powerful function approximators, learning from partial observations still requires introducing prior knowledge. This challenge has two parts:

(1) \textit{Explicit prior modeling is hard}: the prior must be both flexible and tractable. Structural causal models can encode Markovian properties~\cite{pearl2009causality, spirtes1991algorithm,spirtes2000causation}, but they often rely on independence testing assumptions and scale poorly, so expert knowledge is typically required in highly structured settings~\cite{pearl2009causality}.

(2) \textit{Uncertainty over prior knowledge is harder}: the bottleneck is \emph{specifying} the right prior, which becomes a near 'chicken-and-egg' problem when manually encoding what is known vs.\ unknown~\cite{ot4sl}.

We show that, despite this need for manual prior design in highly structured settings~\cite{pearl2009causality}, a PFN-style ``causal ladder'' view~\cite{hollmann2022tabpfn} motivates a level-1.5 \emph{in-model} ``almost-free-lunch'': not unconstrained causal discovery, but amortized identification of missingness/causal dependencies for prediction under a second-order SCM prior~\cite{balazadeh2025causalpfn, robertson2025dopfn}.

\textbf{Challenge 3: Dealing with Uncertainty for Structural Missingness.} Traditional approaches often rely on \emph{deterministic} imputation to fill in missing values. However, \emph{even if} one had a reasonable prior/mechanism in mind, plug-in imputation does not fully leverage the prior/posterior support. Instead of integrating over significant possible missing values in the downstream task, it commits to a single fill-in and loses uncertainty propagation.

In conclusion, under structural missingness it is difficult to specify or integrate prior knowledge that faithfully captures the dependencies between missing and observed/unobserved variables across arbitrary tasks, which is named \textbf{plug-in imputation bias}, leading to suboptimal performance. The question leaves:

\begin{quote}
    \emph{How can we address the three challenges in practice in a single, two-birds-with-one-seed unified efficient framework?}
\end{quote}

\textbf{Our Approach.} \emph{Bayesian inference is both necessary and sufficient} for the above challenges: to be correct under MNAR and avoid plug-in imputation bias, a predictor must reason about the \emph{uncertainty} of unobserved variables, rather than committing to a single imputed input, after which optimal prediction follows from standard Bayesian principles as demonstrated in Sec.~\ref{sec:bayesian_inference}. We make this view practical by constructing a prior distribution over the missingness model and causal mechanisms in Sec.~\ref{sec:posterior_distribution}, and learning an implicit posterior-predictive map from incomplete inputs, with an optional missing-value posterior for sampling, in Sec.~\ref{sec:fitting_ppd_pd}.

We formalize the task and derive the Bayesian predictive formulation in Sec.~\ref{sec:problem_definition} and Sec.~\ref{sec:bayesian_inference}, and discuss near Bayes-optimality via a finite-sample excess-risk bound in Sec.~\ref{sec:analysis_in_context}. Finally, we validate the method in two real-world experimental suites for classification and imputation in Sec.~\ref{sec:experiments}. For conceptual background, we provide a comprehensive related-work discussion in Sec.~\ref{sec:related_work}.

\textbf{Contributions.} 

\begin{enumerate}
    \item We formulate classification with structural missingness as posterior predictive inference conditioned on the observed subspace and the mask, unifying MCAR/MAR/MNAR (\textbf{addresses Challenge 2}).
    \item We propose posterior integration for prediction via a PFN (nearly optimal under the prior, \textbf{addresses Challenge 3}), with an optional learned missing-value posterior for sampling/analysis, propagating uncertainty and avoiding plug-in imputation bias and MNAR bias (\textbf{addresses Challenge 1}).
    \item We give a finite-sample risk bound that decomposes posterior-approximation and predictor errors, explaining plug-in bias and MNAR bias and implying lower sample complexity to reach a target error than plug-in/imputation baselines.
\end{enumerate}
Code and scripts will be released upon acceptance.

\section{Problem Definition}\label{sec:problem_definition}

We consider a training dataset $D = \{(x_i, y_i)\}_{i=1}^N$, where each $x_i \in \mathbb{R}^d$ represents a feature vector and $y_i \in \{1, \ldots, C\}$ is the corresponding class label. The feature vectors are subject to structural missingness, denoted by a binary mask $M_i \in \{0, 1\}^d$, where $M_{ij} = 1$ indicates that the $j$-th feature of the $i$-th sample is missing, and $M_{i} = 0$ indicates it is observed. Our goal is (1) to make tractable inference for the missing variables $X_m$ given the observed variables $X_m^c$ and the mask $M$ and (2) to learn a predictor $p(y\mid x, D)$ for a new test feature vector $x$ governed by a specific mask $m$, utilizing the information from the training dataset $D$ and its associated collection of masks $\{M_i\}_{i=1}^N$.

\textbf{Example (classification with missing features).} Consider binary disease prediction from an incomplete patient record. Each patient has a latent complete feature vector $x\in\mathbb{R}^d$ (e.g., labs and survey answers) and label $y\in\{0,1, \dots k\}$, but at test time we only observe $(x^c_m, m)$. The task is to output a calibrated classifier for the label conditioned on what is actually observed under the given structural missingness pattern.

In the next section (Sec.~\ref{sec:bayesian_inference}), we give the Bayesian formulation of this problem definition by expressing prediction as posterior predictive inference conditioned on the observed subspace and the mask.

\begin{figure*}[thbp]
    \centering
    \includegraphics[width=1\linewidth]{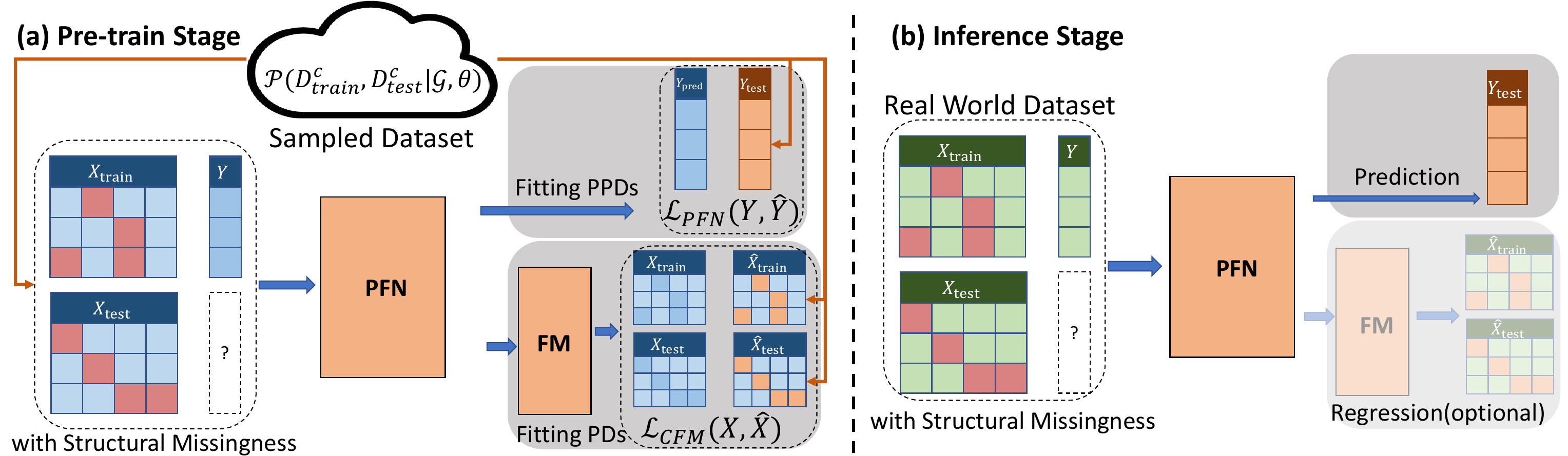}
\caption{Overview of the training and inference framework. In pre-training, inputs $(X_{\text{train}},X_{\text{test}})$ are incomplete (with masks $M$), sampled from the causal model in Sec.~\ref{sec:causal_model} and Fig.~\ref{fig:dfsm}. The PFN is trained end-to-end to output the \emph{in-model} posterior predictive distribution (PPD) directly from incomplete inputs (i.e., it implicitly marginalizes over missing values under the SCM prior), and predicts $y_{\text{test}}$ for the corresponding incomplete $X_{\text{test}}$. In addition, a Flow Matching head can be trained to model a PD over missing values for explicit sampling/imputation and uncertainty analysis, but label prediction uses the PFN output and does not require Monte Carlo over completions.}
    \label{fig:modelstructure}
\end{figure*}

\section{Bayesian Inference with Structural Missingness}\label{sec:bayesian_inference}

We define the missing variable $X_m$ and the observed variable $X_m^c$ as stochastic projections of the complete data dependent on the realization of $M$. Let $\text{Proj}: (M, X) \mapsto X_m$ denote the projection operator that maps a vector to the subspace indexed by set $S$. We formally define:

\begin{equation}
    \begin{aligned}
X_m & := \text{Proj}(M, \mathbf{X}) \\
X_m^c & := (\text{Proj}(M_c, \mathbf{X}), M)
    \end{aligned}
\end{equation}

Note that $X_m^c$ explicitly includes the mask $M$, ensuring that the "observation" includes information about which variables are observed. $X_m$ and $X_m^c$ are still well-defined random variables to perform Bayesian inference. The formal proof that $X_m$ and $X_m^c$ constitute valid random variables on the appropriate disjoint union spaces is provided in Appendix~\ref{pf:defxm}.

Consider a dataset of $n$ independent and identically distributed samples, partitioned into observed and missing components: $D_m^c = \{ (X_{m,i}^c, y_i) \}_{i=1}^n$ and $D_m = \{ X_{m,i} \}_{i=1}^n$. We assume a prior distribution over the missing parameters (or latent variables). The inference goal is to compute the posterior predictive distribution of the target $y$ given the observed history $D_m^c$ and the current observation $X_m^c$.

\textbf{\underline{P}redictive \underline{P}osterior \underline{D}istribution (PPD).}  Thus the regular conditional probability formula is still valid for the calculation of the predictive distribution. The predictive distribution is obtained by marginalizing over the \emph{posterior distribution} of the missing data:

\begin{equation}
\begin{aligned}
    & p(y|X_m^c, D_m^c) = \\
    & \int_{X_m, D_m} \underbrace {p(y| X_m^c, X_m, D_m^c, D_m)}_{\text{General Situation}} d \underbrace{P(X_m, D_m |X_m^c, D_m^c)}_{\text{PD of Missing Data}}
\end{aligned}
\label{eq:ppd-decomposition}
\end{equation}

This formulation allows us to decompose the complex problem of structural missingness into two distinct, manageable components as shown in Eq.~\eqref{eq:ppd-decomposition}. 

\textbf{Predictive Distribution without missingness.} The first term, $p(y| X_m^c, X_m, D_m^c, D_m)$, represents the \textit{general situation}: the standard predictive task where complete data is available. Since $X_m^c$ and $X_m$ together reconstruct the full feature vector $\mathbf{X}$ (and similarly for the dataset $D$), this term reduces to the likelihood of the label given complete features, which is the domain of standard supervised learning models.

\textbf{\underline{P}osterior \underline{D}istribution (PD) of the missing data.} The second term, $P(X_m, D_m |X_m^c, D_m^c)$, represents the \textit{posterior distribution (PD) of the missing data}. This term encapsulates the uncertainty regarding the unobserved features given the observed evidence and the specific missingness pattern. Evaluating the predictive distribution therefore hinges on our ability to model and making more accurate inference on this posterior.

Note that the naive same-distribution shortcut discussed in \textbf{Challenge 1} can be viewed as replacing this posterior with one that ignores the mask information, effectively treating missing values as if drawn from the observed-value distribution (e.g., $P(x\mid M=0)$) rather than the MNAR-correct distribution for missing entries (e.g., $P(x\mid M=1)$). Eq.~\eqref{eq:ppd-decomposition} makes explicit that such a shortcut changes the integral measure and can bias the resulting predictive distribution.

\begin{figure*}[t]
    \centering
    \includegraphics[width=1\linewidth]{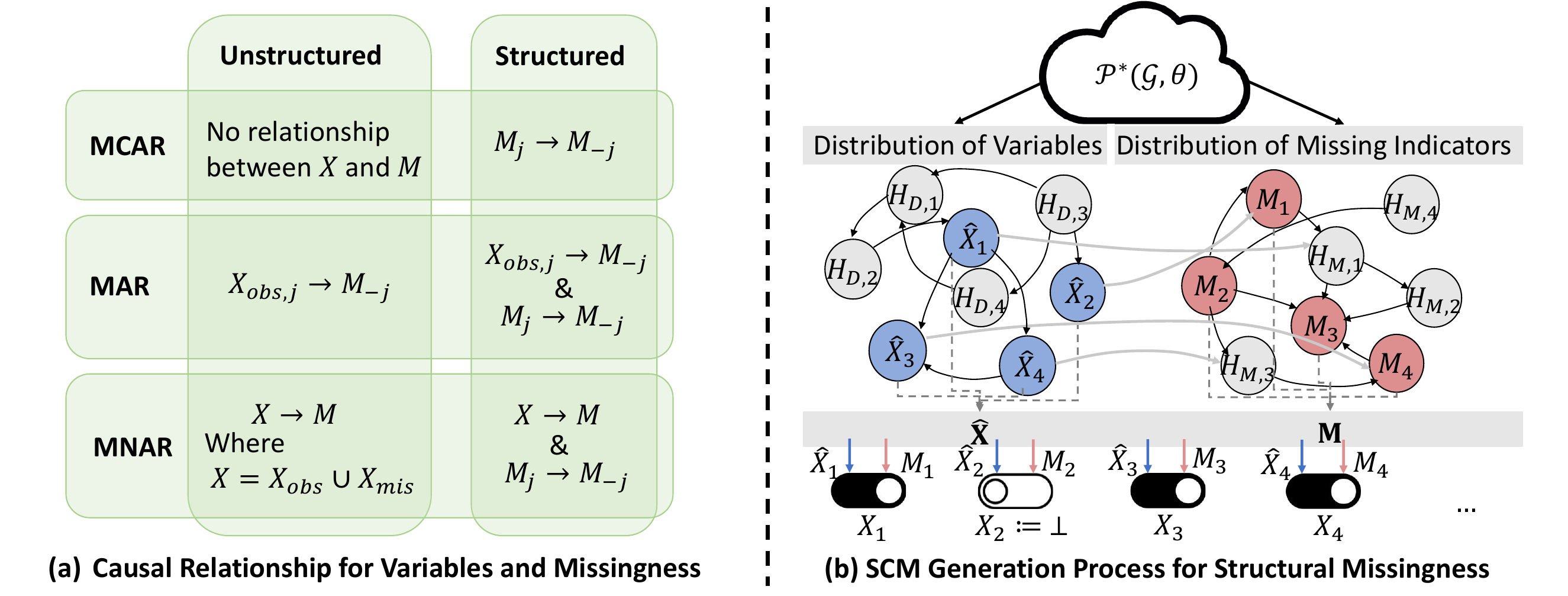}
    \caption{Second-order SCMs for structural missingness. (\textbf{Left}) A causal view of structural missingness, capturing both $X \rightarrow M$ (structural invalidity/undefined values) and $M \rightarrow M$ (missingness propagation) dependencies. (\textbf{Right}) The second-order SCM generation pipeline: sample a causal graph and parameters, generate complete data, then generate masks via the missingness mechanism; the resulting prior is used to fit PFNs and to learn the missing-value posterior.}
    \label{fig:dfsm}
\end{figure*}

However, making both parts tractable, and computing the required posterior predictive integral, is challenging. In the following sections, we derive a tractable approximation to the posterior by exploiting dependencies between observed and missing subspaces, and learn how to compose SCM-based data-generation and missingness mechanisms to better approximate the posterior predictive distribution.

\section{Posterior Distribution of Structural Missingness}\label{sec:posterior_distribution}

To model the structural missingness patterns, we employ second-order structural causal models (SCMs) (Figs.~\ref{fig:modelstructure} and~\ref{fig:dfsm}). This SCM defines the task prior whose induced posteriors we approximate: the observation channel $P(X_m^c\mid X_m)$ and the label map $P(y\mid X_m,D_m)$. PFNs amortize Bayes inference under this prior; if the prior is misspecified, the learned predictive is the best approximation within the model class in expected conditional KL (Theorem~\ref{thm:pfn-risk}). The two fitting routes are introduced next in Sec.~\ref{sec:fitting_ppd_pd}.

\subsection{Causal Relationships for the Missing Indicators}\label{sec:causal_model}

Before formalizing the full generative model, we clarify the causal structure of the missingness indicators $\mathbf{M}$. Let $\mathbf{X}$ be the complete data matrix and $\mathbf{M}$ the binary mask, where $M_{ij}=1$ indicates that $X_{ij}$ is missing. Classical settings distinguish MCAR/MAR/MNAR depending on whether $\mathbf{M}$ is independent of $\mathbf{X}$, depends on observed entries $X_{obs}$, or additionally depends on unobserved entries $X_{miss}$. In our scope, SM refers to structured (Markovian) dependence: the missingness mechanism (including the gate) is treated as a random variable under a second-order prior, while the induced Markov structure of the SCM is fixed.

Importantly, missingness can also be \emph{structured} within $\mathbf{M}$ itself (Fig.~\ref{fig:dfsm}). The missingness of one field may force others to be missing, yielding edges $M_j \rightarrow M_{-j}$ (e.g., skipping ``Smoker?'' implies ``Cigarettes per day'' is undefined). Hence, a causal model for SM must capture both (1) $X \rightarrow M$ (state-dependent invalidity) and (2) $M \rightarrow M$ (propagation of missingness).

\subsection{Constructing SCMs for Structural Missingness}

We model SM with a Structural Causal Model (SCM) $\mathcal{M}=\langle \mathbf{X},\mathbf{U},\mathcal{F},P(\mathbf{U})\rangle$ and augment it with an explicit generator for the missingness mask $\mathbf{M}$. For each variable $X_i$, the missingness indicator is produced by
$M_i \leftarrow f_{M_i}(PA(M_i), U_{M_i})$,
where under SM, $f_{M_i}$ is a deterministic gate that can depend on parent values ($X\rightarrow M$) and other indicators ($M\rightarrow M$). In our implementation we use a randomized score-and-quantile gate; details are in Appendix~\ref{sec:impl:scm:missingness-prior}. We further use a \textit{second-order} SCM by treating the graph $\mathcal{G}$, parameters $\theta$, and missingness logic as latent random variables. The resulting generative process is:
\begin{enumerate}
    \item Sample a causal structure and parameters: $(\mathcal{G}, \theta) \sim P(\mathcal{G}, \theta)$.
    \item Sample exogenous noise: $\mathbf{U} \sim P(\mathbf{U})$.
    \item Generate complete data $\mathbf{X}$ via structural equations $X_i \leftarrow f_i(PA_i, U_i)$.
    \item Generate missingness indicators $\mathbf{M}$ via $M_i \leftarrow f_{M_i}(PA(M_i))$, capturing both structural constraints ($X \rightarrow M$) and propagated missingness ($M \rightarrow M$).
    \item Apply mask to obtain observed data: $X_{obs} = \mathbf{X} \odot (1-\mathbf{M}) + \bot \odot \mathbf{M}$.
\end{enumerate}

Our goal is the posterior induced by this SCM prior. In practice, we learn two prior-generated targets: the structural-missingness observation channel $P(X_m^c\mid X_m)$ and the label map $P(y\mid X_m,D_m)$, which motivates the two learned objects described next. See Appendix~\ref{sec:related_work} to learn concretely about PFNs.

\section{Fitting the PPDs and PDs for Structural Missingness}\label{sec:fitting_ppd_pd}

Given the decomposition in Eq.~\eqref{eq:ppd-decomposition}, our inference strategy requires estimating two key objects: the posterior predictive for labels, and (optionally) a posterior distribution (PD) over missing values. In this section, we detail our two-pronged approach: employing Prior Fitted Networks (PFNs) to learn $p(y\mid D_m^c,X_m^c)$ end-to-end from tasks generated by the SCM prior defined in Sec.~\ref{sec:posterior_distribution}, and utilizing Flow Matching as an auxiliary head to model the missing-value posterior conditioned on the same incomplete context.

\subsection{Bayesian Inference with Prior Fitted Networks}

Prior Fitted Networks (PFN) \cite{muller2021transformers} are a class of neural networks designed to approximate Bayesian inference by learning a mapping from datasets to posterior predictive distributions (PPD). PFNs are trained on a large number of synthetic datasets generated from a prior distribution over models, allowing them to generalize to new datasets and perform inference efficiently.

In our framework, we treat the Second-Order SCM described in the previous section as the prior. A PFN, denoted as $\phi$, is trained to approximate the posterior predictive distribution of the target label $y$ given the observed data $D_m^c$ and query $X_m^c$. Specifically, the PFN minimizes the cross-entropy loss between the predicted distribution and the true labels on synthetic tasks generated from the SCM prior:
\begin{equation}
    \begin{aligned}
    & \mathcal{L}_{PFN}(\phi) = \\
    & \mathbb{E}_{(\mathcal{M}, D) \sim P(\mathcal{M}, D)} \left[ - \sum_{(x,y) \in D_{test}} \log P_\phi(y | x, D_{train}) \right]
    \end{aligned}
\end{equation}
where $(\mathcal{M}, D)$ is a dataset sampled from the second-order SCM prior. In our setting, the PFN consumes incomplete inputs $(D_m^c,X_m^c)$ (including masks) and directly outputs the posterior predictive $p_\phi(y\mid D_m^c,X_m^c)$; the ``general situation'' is the special case with no missingness. This learned map implicitly marginalizes over latent missing values (and missing context $D_m$) under the SCM prior in a single forward pass.

\subsection{Flow Matching Regressor for the PDs of Missing Variables}

While the PFN already targets the required marginal posterior predictive $p_\phi(y\mid D_m^c,X_m^c)$ from incomplete data, we also learn a posterior over missing values for imputation and uncertainty analysis. Concretely, we train a Conditional Flow Matching (CFM) head to approximate $Q_\theta(X_m\mid X_m^c, D_m^c)$, where $D_m$ is not represented explicitly but is marginalized through the same task prior that generates incomplete contexts. In contrast to MSE regression (which only learns a conditional mean), CFM can represent multimodal posteriors induced by structural constraints.

To address this, we propose using a \textbf{Flow Matching Regressor}. Inspired by Conformal Regression \cite{lei2018distribution}, our approach seeks to provide rigorous uncertainty quantification. However, unlike standard conformal methods which often focus on interval calibration, or MSE-based regressors that ignore the distributional shape, we aim to learn a \textit{nearly KL-optimal} regression.

We employ Conditional Flow Matching (CFM) \cite{lipman2023flow} to learn a continuous normalizing flow that transforms a simple base distribution (e.g., Gaussian) into the complex posterior distribution of the missing variables conditioned on the observed ones. The objective is to minimize the regression loss between the vector field generated by the neural network $v_t(x, \theta)$ and the target vector field $u_t(x|x_1)$ defined by the probability path between the noise and data:
\begin{equation}
    \mathcal{L}_{CFM}(\theta) = \mathbb{E}_{t, q(x_1), p_t(x|x_1)} ||v_t(x) - u_t(x|x_1)||^2
\end{equation}
During pretraining, we balance this regression objective with the PFN cross-entropy using a fixed weight (CFM loss weight $0.1$) and do not tune it extensively. The CFM head is conditioned on the same incomplete inputs $(D_m^c,X_m^c,M)$ via the shared PFN-style backbone. In our pipeline, prediction does not require Monte Carlo over $X_m$ or $D_m$: the PFN outputs $p_\phi(y\mid D_m^c,X_m^c)$ directly, while the flow head is used for explicit samples/imputations from $Q_\theta$.
Concrete pre-training hyperparameters (SCM prior distributions, missingness-generator settings, PFN backbone, and flow head) are provided in Appendix~\ref{sec:impl:pretraining}.

\section{Analysis of Bias Alleviation and Near Optimality During In-Context Learning}\label{sec:analysis_in_context}

In this section we make precise near optimality in practice and why it is theoretically justified. The argument follows the subsection structure: 

(1) \emph{Risk minimization and approximation} establishes Bayes-near-optimality, with task-averaged training as an expected \emph{conditional KL projection} (Theorem~\ref{thm:pfn-risk}) and \emph{consistency} for the missing-value posterior (Theorem~\ref{thm:cfm-consistency}). As properties for methods in Secs.~\ref{sec:posterior_distribution} \& \ref{sec:fitting_ppd_pd}. 

(2) \emph{Bias alleviation} formalizes how posterior integration controls error via \emph{posterior-mismatch} and \emph{conditional prediction} terms (Theorem~\ref{thm:post-int}), explaining plug-in failures via the \emph{Jensen gap} (Corollary~\ref{cor:jensen-gap}) and unavoidable \emph{MNAR mismatch} (Corollary~\ref{cor:forced-same-dist-mismatch}, Remark~\ref{rem:mnar-ident}). As previously introduced as \textbf{Challenges 1} \& \textbf{3}.

and (3) \emph{Sample complexity} for near-optimality, where the decoupled design is sufficiently non-biased and more \emph{sample-efficient} than end-to-end learning (Theorem~\ref{prop:decoupled-sample}).

\paragraph{Notation.}
When the conditioning context $(X_m^c, D_m^c)$ is clear, we abbreviate the true missing-value posterior and its learned approximation as
\begin{equation*}
    \mu := P^\star(\cdot \mid X_m^c, D_m^c),
    \qquad
    \nu := Q_\theta(\cdot \mid X_m^c, D_m^c).
\end{equation*}

With $\phi$ and $\theta$ denoting the PFN (predictive) and flow (posterior) parameters, respectively.

\subsection{Risk Minimization and Approximation}\label{sec:rma}

To connect the model construction above with the bias and near-optimality results that follow, we first pin down what is learned in the population limit for the two components used in posterior integration.
\begin{quote}
    \emph{How can we fit the posterior predictive distribution and the missing-value posterior nearly optimally, and when do these learned objects justify posterior integration?}
\end{quote}
In our instantiation, the posterior predictive is learned by a Prior-Fitted Network and the missing-value posterior is learned by conditional flow matching.

\begin{theorem}[PFN as population risk minimization of posterior predictives]\label{thm:pfn-risk}
Let $\Pi$ be the second-order SCM prior over tasks (data-generating models). A task sampled from $\Pi$ induces a random dataset $D$ and, for any query $x$, a true posterior predictive distribution over labels
$P^\star(\cdot\mid x,D)$.
Let $\{P_\phi(\cdot\mid x,D):\phi\in\Phi\}$ be the PFN model class and consider the population cross-entropy risk
\begin{equation*}
    \begin{aligned}
    \mathcal{R}(\phi) :=
    \mathbb{E}_{(\mathcal{M},D)\sim \Pi}\,
    \mathbb{E}_{x, y\sim P^\star(\cdot, \cdot \mid D)}\,
    \big[-\log P_\phi(y\mid x,D)\big].
    \end{aligned}
\end{equation*}
Then:
\begin{enumerate}
    \item[(i)] (Bayes optimality) Any minimizer $\phi^\star\in\arg\min_{\phi\in\Phi}\mathcal{R}(\phi)$ also minimizes the expected conditional KL divergence
\end{enumerate}
    \begin{equation*}
        \mathbb{E}_{(\mathcal{M},D)\sim\Pi}\,\mathbb{E}_{x\sim P^\star(\cdot\mid D)}
        \Big[\mathrm{KL}\big(P^\star(\cdot\mid x,D)\,\|\,P_\phi(\cdot\mid x,D)\big)\Big],
    \end{equation*}
\begin{enumerate}
    \item[] and the minimum value equals the Bayes entropy term plus the minimum expected conditional KL.
    \item[(ii)] (Realizable case) If there exists $\bar\phi\in\Phi$ such that $P_{\bar\phi}(\cdot\mid x,D)=P^\star(\cdot\mid x,D)$ for $\Pi$-almost every $(x,D)$, then any population risk minimizer satisfies $P_{\phi^\star}(\cdot\mid x,D)=P^\star(\cdot\mid x,D)$ $\Pi$-a.s.
    \item[(iii)] (Misspecified case) Without realizability, $\phi^\star$ is the best approximation in $\Phi$ in the above expected conditional KL sense.
\end{enumerate}
This is proved in the Appendix Sec.~\ref{pf:pfn-risk}.
\end{theorem}

\begin{theorem}[Consistency of conditional flow matching in \cite{zhou2025error}]\label{thm:cfm-consistency}
Let $P^\star(X_m\mid X_m^c)$ be the true conditional distribution of missing variables.
Under standard regularity conditions for conditional flow matching (existence/uniqueness of the probability path and sufficient model capacity), the population conditional flow matching objective is minimized when the learned conditional distribution equals the target, i.e., $Q_\theta(\cdot\mid X_m^c)=P^\star(\cdot\mid X_m^c)$ almost surely in $X_m^c$.
Moreover, with increasing data and a consistent optimizer, empirical minimizers converge to a population minimizer, yielding $Q_\theta(\cdot\mid X_m^c)\to P^\star(\cdot\mid X_m^c)$ in an appropriate weak sense (e.g., $W_1$) under additional moment conditions.
\end{theorem}

\subsection{In-Model Reduction of Plug-in Imputation and MNAR Bias}\label{sec:bias-reduction}

To highlight the MNAR bias in \textbf{Challenge 1}, write $P(x\mid M=0)$ and $P(x\mid M=1)$ for the observed- and missing-value distributions. Under MNAR, $P(x\mid M=1)$ can differ from $P(x\mid M=0)$, so treating missingness as ignorable or using point imputation can introduce systematic bias. Our results show that posterior integration admits an explicit error bound (Theorem~\ref{thm:post-int}), point imputation can incur a strict Jensen gap (Corollary~\ref{cor:jensen-gap}), and any imputer forced to use one distribution across missingness groups suffers unavoidable mismatch when $P(x\mid M=1)\neq P(x\mid M=0)$ (Corollary~\ref{cor:forced-same-dist-mismatch}).

Let $d$ be a metric on $\mathcal{X}_m$, and let $W_1$ be the 1-Wasserstein distance induced by $d$. The conditional complete-data target is written as $h^\star(x_m) := \mathbb{E}\!\left[g(Y)\mid X_m=x_m, X_m^c, D_m^c\right]$ and the predictor is written as $\hat h_\phi$ (e.g., PFN).

\begin{theorem}[Posterior integration converges with no in-model bias]\label{thm:post-int}
For any scalar test function $g$, define the true conditional target
\begin{equation}
    \begin{aligned}
    T_g^\star(X_m^c, D_m^c) 
    &:= \mathbb{E}\!\left[g(Y)\mid X_m^c, D_m^c\right] \\
    &= \mathbb{E}_{X_m\sim \mu}\!\left[h^\star(X_m)\right],
    \end{aligned}
\end{equation}
and the posterior-integration estimator
\begin{equation}
    \widehat T_g(X_m^c, D_m^c)
    := \mathbb{E}_{X_m\sim \nu}\!\left[\hat h_\phi(X_m)\right].
\end{equation}
Under Assumption~\ref{assump:post-int},
\begin{equation}
    \left|T_g^\star(X_m^c, D_m^c) - \widehat T_g(X_m^c, D_m^c)\right|
    \le L\,\varepsilon_{\text{post}} + \varepsilon_{\text{pred}}.
\end{equation}
This is proved in the Appendix Sec.~\ref{pf:post-int}.
\end{theorem}

\begin{corollary}[Strict Jensen gap of point imputation]\label{cor:jensen-gap}
Let $\mu:=P^\star(\cdot \mid X_m^c, D_m^c)$ denote the true conditional distribution. If $h^\star(\cdot)$ is strictly convex or strictly concave in $x_m$, $\mu$ is non-degenerate, and $\mathbb{E}_{X_m\sim \mu}[X_m]$ exists, then the point-imputation plug-in predictor
\begin{equation}
    \widehat T_{g,\text{point}} := h^\star\!\left(\mathbb{E}_{X_m\sim \mu}[X_m]\right)
\end{equation}
incurs a strict Jensen gap:
\begin{equation}
    \left|\mathbb{E}_{X_m\sim \mu}[h^\star(X_m)] - h^\star\!\left(\mathbb{E}_{X_m\sim \mu}[X_m]\right)\right| > 0.
\end{equation}
This is proved in the Appendix Sec.~\ref{pf:jensen-gap}.
\end{corollary}


\begin{corollary}[Inevitable mismatch under a forced same-distribution imputer]\label{cor:forced-same-dist-mismatch}
Let $P^\star$ denote the true data-generating distribution, and define the conditional laws
$P_0:=P^\star(X_m\mid M=0)$ and $P_1:=P^\star(X_m\mid M=1)$.
Suppose $P_0\neq P_1$ but an imputation/generative model is constrained to satisfy
$R(X_m\mid M=0)=R(X_m\mid M=1)=:\widetilde P$.
Then no choice of $\widetilde P$ can match both $P_0$ and $P_1$ simultaneously; in particular,
\begin{equation}
\max\big\{\mathrm{TV}(P_0,\widetilde P),\mathrm{TV}(P_1,\widetilde P)\big\}\;\ge\;\tfrac12\,\mathrm{TV}(P_0,P_1).
\end{equation}
Moreover, taking the supremum over bounded measurable $f$ with $\|f\|_\infty\le 1$,
\begin{equation}
\begin{aligned}
\sup_{\|f\|_\infty\le 1}\max\Big\{\big|\mathbb{E}_{P_0}f-\mathbb{E}_{\widetilde P}f\big|,\ \big|\mathbb{E}_{P_1}f-\mathbb{E}_{\widetilde P}f\big|\Big\} \\
\;\ge\;\mathrm{TV}(P_0,P_1),
\end{aligned}
\end{equation}
and hence at least one group ($M=0$ or $M=1$) incurs a non-vanishing distribution-mismatch bias whenever MNAR induces $P_0\neq P_1$.

See Appendix Sec.~\ref{pf:forced-same-dist-mismatch} for the proof.
\end{corollary}

\begin{remark}[MNAR identifiability is model-dependent]\label{rem:mnar-ident}
Theorem~\ref{thm:post-int} does not claim that MNAR is identifiable without assumptions. Rather, it formalizes an in-model inference statement: if the assumed SCM prior (including $X\!\to\! M$ and $M\!\to\! M$ structure) renders the posterior $P^\star(X_m\mid X_m^c,D_m^c)$ identifiable within the model class, then learning that posterior and integrating it avoids systematic plug-in bias in prediction.
\end{remark}

\subsection{Sample Complexity Advantage of Decoupled Inference}\label{sec:complexity}

The practical question is:
\begin{quote}
\emph{Given only the observed features $x_m^c$, and limited context $D_m^c$ at test time, how can we predict $y$ as close as possible to the Bayes-optimal answer?}
\end{quote}
Theorem~\ref{prop:decoupled-sample} gives a sample-complexity advantage (ICL difficulty proxy) for learning $G$ (missing-feature posterior) and $H$ (conditional predictor) separately rather than end-to-end in finite-context / ICL. Appendix Sec.~\ref{pf:postpred-decoupled} provides the supporting error decompositions and proof chain (Theorem~\ref{prop:decoupled-bayes-proxy}; Corollary~\ref{cor:postpred-two-term}).

\begin{theorem}[Sample Complexity Advantage of Decoupled Inference (ICL Difficulty Proxy)]\label{prop:decoupled-sample}
Consider the Bayes posterior-predictive target map
$F(x_m^c,D_m^c):=P^\star(\cdot\mid x_m^c,D_m^c)$.
Equivalently, a ``decoupled'' view learns (i) a missing-feature posterior generator
$G(\cdot\mid x_m^c,D_m^c)\approx P^\star(\cdot\mid x_m^c,D_m^c)$
and (ii) a conditional predictor $H(\cdot\mid x_m,x_m^c,D_m^c)\approx P^\star(\cdot\mid x_m,x_m^c,D_m^c)$,
with the (conceptual) posterior-integrated predictor
$\widehat P(y\mid x_m^c,D_m^c):=\int H(y\mid x_m,x_m^c,D_m^c)\,G(x_m\mid x_m^c,D_m^c)\,dx_m$.

Assume the conditional predictor regularity in Appendix Sec.~\ref{pf:postpred-decoupled} so that the posterior-mismatch term is controlled by $L_h\,W_1(\cdot,\cdot)$ (Lemma~\ref{lem:postpred-w1}).
Then Corollary~\ref{cor:postpred-two-term} implies that to reach total error $\epsilon$ it suffices to learn
$G$ to Wasserstein accuracy $\epsilon/(2L_h)$ and $H$ to conditional-prediction error $\epsilon/2$, capturing the required error propagation.

Further, using an \emph{effective complexity proxy} viewpoint, assume learning a $d$-dimensional target to error $\epsilon$ scales as
\begin{equation*}
    \mathcal{N}(L,\epsilon)=\mathcal{O}\big((L/\epsilon)^d\big),
\end{equation*}
and let $L_{\mathrm{cpl}}$ denote an \emph{ICL coupling penalty} (the added difficulty of learning the coupled posterior-integrated map directly from limited context). Then the end-to-end sample complexity $\mathcal{N}_{E2E}$ versus the decoupled complexity $\mathcal{N}_{Decoupled}$ admits the comparison
\begin{equation*}
    \begin{aligned}
    \mathcal{N}_{E2E}
    &\approx \mathcal{O}\Big(\big((L_h(1+L_g)+L_{\mathrm{cpl}})/\epsilon\big)^d\Big),\\
    \mathcal{N}_{Decoupled}
    &\approx \mathcal{O}\Big(\big((L_gL_h)/\epsilon\big)^d\Big)+\mathcal{O}\Big(\big(L_h/\epsilon\big)^d\Big),
    \end{aligned}
\end{equation*}
Thus, when $L_g$ is large and context is limited (large $L_{\mathrm{cpl}}$), learning $G$ and $H$ separately can be more sample-efficient than learning the coupled end-to-end map; see Appendix Sec.~\ref{pf:postpred-decoupled}.
\end{theorem}

\section{Experiments}\label{sec:experiments}

\FloatBarrier
\begin{figure*}[!t]
  \centering
  \includegraphics[width=\textwidth,height=0.32\textheight,keepaspectratio]{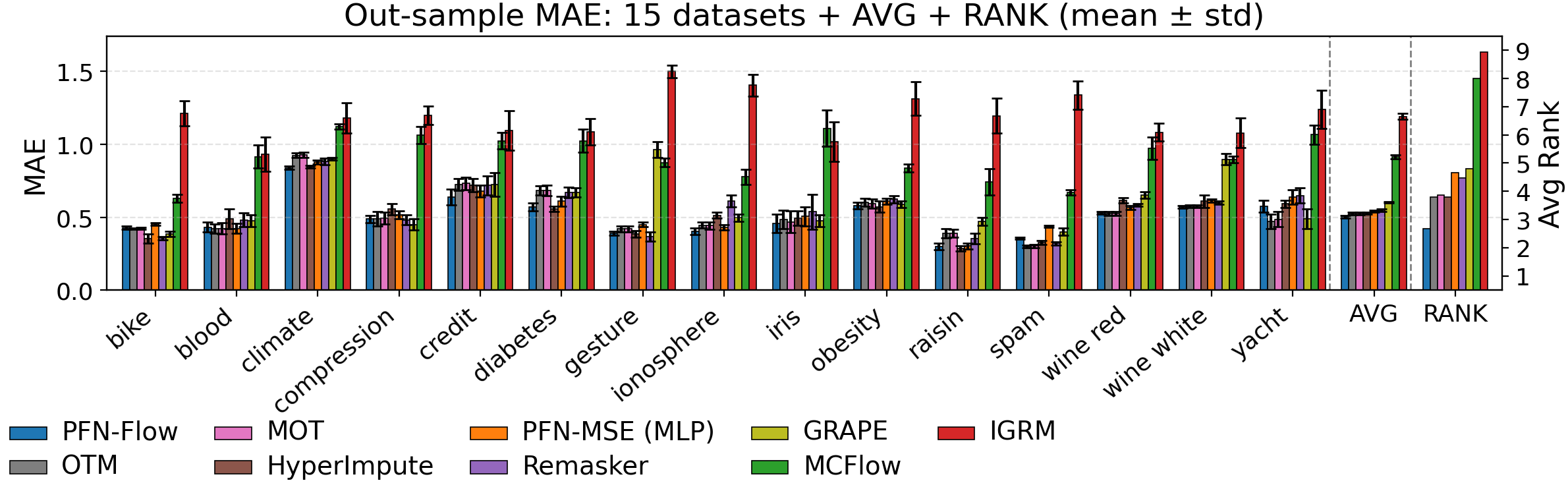}
  \caption{
  Out-of-sample MAE under MNAR (mean$\pm$std over $K=10$ independently sampled MNAR mask realizations) across $N$ datasets (here $N=15$ from the cross-method intersection),
  with dataset-wise grouped bars, an aggregated AVG column, and an Avg Rank column on a secondary y-axis (lower is better).
  }
  \label{fig:outsample-mae-mnar}
  \vspace{2mm}
  \noindent\leavevmode
\begin{minipage}{0.48\textwidth}
\centering
\scriptsize
\setlength{\tabcolsep}{3pt}
\renewcommand{\arraystretch}{0.95}
\captionof{table}{Method-level performance summary. Missing rate $\le 10\%$. AUC is reported as mean$\pm$std across datasets. \textbf{Rank} denotes average per-dataset rank. \textbf{Time} is reported as mean$\pm$std(seconds) per 1000 instance with 50 variables. Datasets: 33. Showing top-14 + TabPFN (Raw) of 50 methods. Full results in Table~\ref{tab:auc-method-summary-0-10-all}.}
\label{tab:auc-method-summary-0-10}
\resizebox{\linewidth}{!}{%
\begin{tabular}{lrrr}
\toprule
\multicolumn{1}{c}{Method} & \multicolumn{1}{c}{AUC} & \multicolumn{1}{c}{Rank} & \multicolumn{1}{c}{Time (s)} \\
\midrule
PFN-Flow & $0.8694\pm0.1448$ & 16.83 & $0.061\pm0.003$ \\
CatBoost + TDM & $0.8664\pm0.1512$ & 20.77 & $160.830\pm8.329$ \\
PFN-NSM & $0.8663\pm0.1499$ & 18.70 & $0.061\pm0.003$ \\
CatBoost + GRAPE & $0.8660\pm0.1521$ & 21.65 & $43.530\pm0.294$ \\
CatBoost + IGRM & $0.8659\pm0.1513$ & 21.44 & $44.152\pm0.241$ \\
CatBoost + OTM-SK & $0.8657\pm0.1514$ & 20.70 & $82.744\pm2.056$ \\
CatBoost + OTM-RR & $0.8656\pm0.1504$ & 22.02 & $42.211\pm0.211$ \\
CatBoost + Remasker & $0.8656\pm0.1501$ & 21.32 & $551.288\pm9.349$ \\
CatBoost + MOT-MLP & $0.8646\pm0.1527$ & 22.94 & $165.097\pm0.798$ \\
CatBoost + MCFlow & $0.8646\pm0.1522$ & 20.91 & $65.533\pm0.612$ \\
CatBoost + HyperImpute & $0.8645\pm0.1522$ & 23.11 & $207.728\pm0.543$ \\
CatBoost (Raw) & $0.8642\pm0.1523$ & 22.17 & $37.031\pm0.186$ \\
CatBoost + MissForest & $0.8641\pm0.1526$ & 23.14 & $63.092\pm1.692$ \\
CatBoost + MOT-LIN & $0.8639\pm0.1527$ & 24.15 & $164.264\pm3.783$ \\
\hdashline
TabPFN (Raw) & $0.8545\pm0.1610$ & 26.73 & $0.061\pm0.003$ \\
\bottomrule
\end{tabular}%
}
\end{minipage}%
\hfill%
\begin{minipage}{0.48\textwidth}
\centering
\scriptsize
\setlength{\tabcolsep}{3pt}
\renewcommand{\arraystretch}{0.95}
\captionof{table}{Missing rate $> 10\%$. Datasets: 10. Showing top-14 + TabPFN (Raw) of 50 methods. Full results in Table~\ref{tab:auc-method-summary-10plus-all}.}
\label{tab:auc-method-summary-10plus}
\resizebox{\linewidth}{!}{%
\begin{tabular}{lrrr}
\toprule
\multicolumn{1}{c}{Method} & \multicolumn{1}{c}{AUC} & \multicolumn{1}{c}{Rank} & \multicolumn{1}{c}{Time (s)} \\
\midrule
PFN-Flow & $0.8172\pm0.1099$ & 11.40 & $0.061\pm0.003$ \\
XGBoost + MOT-LIN & $0.8111\pm0.1217$ & 12.00 & $131.921\pm3.778$ \\
XGBoost + MOT-MLP & $0.8073\pm0.1216$ & 17.60 & $132.754\pm0.776$ \\
PFN-NSM & $0.8065\pm0.1219$ & 16.20 & $0.061\pm0.003$ \\
XGBoost + IGRM & $0.8048\pm0.1127$ & 17.90 & $11.809\pm0.156$ \\
XGBoost (Raw) & $0.8047\pm0.1117$ & 15.50 & $4.689\pm0.022$ \\
CatBoost + MOT-LIN & $0.8042\pm0.1228$ & 22.40 & $164.264\pm3.783$ \\
CatBoost + MOT-MLP & $0.8039\pm0.1249$ & 20.30 & $165.097\pm0.798$ \\
CatBoost (Raw) & $0.8034\pm0.1213$ & 23.70 & $37.031\pm0.186$ \\
XGBoost + GRAPE & $0.8032\pm0.1109$ & 18.50 & $11.187\pm0.229$ \\
XGBoost + OTM-RR & $0.8028\pm0.1206$ & 21.40 & $9.869\pm0.103$ \\
TabPFN + GRAPE & $0.8016\pm0.1301$ & 24.90 & $6.560\pm0.228$ \\
XGBoost + Remasker & $0.8014\pm0.1202$ & 25.70 & $518.946\pm9.348$ \\
TabPFN + MOT-LIN & $0.8011\pm0.1290$ & 23.30 & $127.294\pm3.778$ \\
\hdashline
TabPFN (Raw) & $0.7985\pm0.1269$ & 23.50 & $0.061\pm0.003$ \\
\bottomrule
\end{tabular}%
}
\end{minipage}%

\end{figure*}

\paragraph{Details for Baselines and Benchmarks.}
Our evaluation is produced on 58 tabular datasets with missingness: for classification we report 33 datasets ($\le$10\%) + 10 datasets ($>$10\%), covering 50 methods from 4 predictor families (XGBoost/LightGBM/CatBoost/TabPFN-family) and diverse missing-data handling. Dataset lists are in Appendix Tables~\ref{tab:dataset-summary-0-10}--\ref{tab:dataset-summary-10plus} and full 50-method tables are in Appendix Tables~\ref{tab:auc-method-summary-0-10-all}--\ref{tab:auc-method-summary-10plus-all}. For imputation, we report MNAR robustness on $N=15$ shared datasets with $K=10$ masks (Appendix Table~\ref{tab:imputation-dataset-sizes}); protocol details are in Appendix Sec.~\ref{sec:impl:cls} and Sec.~\ref{sec:impl:imp}.

\paragraph{RQ1: Missing-data classification performance.}\label{sec:exp:cls-auc}
\textbf{Question.} Which method is most accurate and robust across datasets under missingness?
\textbf{Answer.} Tables~\ref{tab:auc-method-summary-0-10} and~\ref{tab:auc-method-summary-10plus} show that \textsc{PFN-Flow} ranks first in both missingness regimes, outperforming TabPFN and PFN-NSM (an ablation using a reduced SCM prior as in Sec.~\ref{sec:causal_model}: masks depend on $X$ ($X\!\to\!M$) but have no cross-field propagation ($M\!\to\!M$)). This supports uncertainty-aware completion under structural missingness; see Appendix Sec.~\ref{sec:impl:cls}.

\paragraph{RQ2: Missing-data classification efficiency.}\label{sec:exp:cls-eff}
\textbf{Question.} How much runtime do we save relative to the strongest non-PFN baseline in the tables?
\textbf{Answer.} We compute the speedup using the Time column in Tables~\ref{tab:auc-method-summary-0-10} and~\ref{tab:auc-method-summary-10plus} (baseline fit+inference, including fitting the imputer and classifier; PFN amortized forward-pass inference, with pre-training cost reported in Appendix Sec.~\ref{sec:impl:pretraining}). In the low-missingness group, \textsc{PFN-Flow} takes $0.061$s while CatBoost+TDM takes $160.830$s, giving a speedup of $160.830/0.061 \approx 2.6\times 10^3$. In the high-missingness group, \textsc{PFN-Flow} takes $0.061$s while XGBoost+MOT-LIN takes $131.921$s, giving $131.921/0.061 \approx 2.2\times 10^3$. The runtime measurement protocol is detailed in Appendix Sec.~\ref{sec:impl:cls}.

\paragraph{RQ3: MNAR imputation robustness.}\label{sec:exp:mnar-oot}
\textbf{Question.} Under MNAR missingness, do we improve out-of-sample imputation accuracy and consistency?
\textbf{Answer.} Fig.~\ref{fig:outsample-mae-mnar} shows that \textsc{PFN-Flow} achieves the lowest out-of-sample MAE and the best overall \textbf{AVG} and \textbf{Avg Rank}. The error bars are small, indicating consistent performance across MNAR masks (a practical robustness check without sweeping MNAR strength). The MNAR protocol and metrics are specified in Appendix Sec.~\ref{sec:impl:imp}.

\section{Conclusion}

We propose near-optimal Bayesian inference under structural missingness. By casting the problem as a Second-Order SCM and using Prior-Fitted Networks with Flow Matching, we obtain an approximation that preserves uncertainty rather than collapsing into a single imputation. Theory supports this view, and experiments show consistent gains over imputation baselines on incomplete tabular data.

\section*{Impact Statement}
This paper presents work whose goal is to advance the field of 
Machine Learning. There are many potential societal consequences 
of our work, none which we feel must be specifically highlighted here.

\bibliography{example_paper}
\bibliographystyle{icml2026}
 \clearpage

\appendix

\makeatletter
\let\addcontentsline\icml@origaddcontentsline
\makeatother

\etocsettocstyle{\noindent\textbf{Appendix Contents}\par\medskip}{}
\begingroup
  \etocsetnexttocdepth{subsection}
  \tableofcontents
\endgroup
\newpage

\section{Related Work}\label{sec:related_work}

\subsection{Structural Missingness}
While classic MCAR/MAR/ MNAR definitions \cite{rubin1976inference} cover probabilistic missingness, structural missingness deals with logically undefined values. Recent works have modeled this using graph-based approaches or separate categories, but often lack a unified Bayesian treatment. \cite{mohan2013missing} proposes and proved that dealing with missing data is a causal inference problem in a unified identification framework. However, due to lack of integration of methods and tools, models with simplified assumptions without considering selection bias caused by not explicitly differentiating $P(x|M=0)$ and $P(x|M=1)$ still dominate the field with plausible performance~\cite{muzellec2020missing, ot4sl, zhang2025diffputer}. Though the structure is not identifiable in general without extra assumptions, PFNs can offer an ``almost-free-lunch'' in terms of \emph{efficient amortized inference within an assumed SCM prior}: when the prior renders the relevant posteriors identifiable \emph{in-model} (Remark~\ref{rem:mnar-ident}), PFNs can approximate these posteriors and enable practical inference, rather than resolving fundamental MNAR non-identifiability without assumptions.

\cite{mitra2023learning} propose a comprehensive partition or break down of existing researches, which divides them into (1) making definitions for potential valuable/redundant parts and unignorable bias, such as \cite{rubin1976inference} first provide types that missing data can be classified into, \cite{muzellec2020missing} puts stress on types of missingness that cause unignorability bias.

The definitions provide a framework to (2) collect data and build models, such as KNN~\cite{pujianto2019k}, Gaussian mixture based~\cite{garcia2010pattern}, and Bayesian methods~\cite{ma2018bayesian}, for missing data imputation. Leveraging advanced mathematical frameworks that realize (1), \cite{muzellec2020missing, ot4sl} propose using optimal transport to measure the distance between the observed and missing data, and generative models such as \cite{zhang2025diffputer} propose using diffusion models to generate missing data. 

The models provide tools to (3) make inference to find out truly valuable information to predict and make decisions with the models~\cite{pinkard2025missing}, and most fundamentally (4) find physical rules and causal facts under general assumptions and minimal constraints~\cite{pearl2009causality}, which brings us back to (1) for constructing more effective modeling. 

By considering this framework, we firstly find it more effective to construct missing mechanisms randomly, and fit a powerful regressor to make nearly Bayesian optimal inference on these missing data, which provides almost free lunch for merging the four parts of the loop in one inference.

\subsection{Prior-Fitted Networks}
PFNs \cite{muller2021transformers, muller2025future} and TabPFN \cite{hollmann2022tabpfn} demonstrate the power of learning Bayesian inference on tabular data. These methods have shown state-of-the-art performance on many small or medium-sized datasets since TabPFN-v2~\cite{hollmann2025accurate}. As concerns about scalability to larger datasets are raised, TabICL~\cite{qu2025tabicl} proposes a distribution-aware feature embedding method to improve scalability, and BETA~\cite{liu2025tabpfn} adapts a lightweight encoder to align TabPFN to large-scale high-dimensional tabular data. As shown in the technical report for TabPFN-2.5~\cite{grinsztajn2025tabpfn}, it achieves dominant performance (87\% win rate) on tabular datasets up to 100K samples and 2K features. PFN-based methods also have differentiable properties that can be leveraged to produce feature-importance measures, partial dependence~\cite{rundel2024interpretable}, and high-dimensional Bayesian optimization~\cite{muller2023pfns4bo}, and show promising results in settings requiring causal mechanisms, such as causal effect estimation~\cite{balazadeh2025causalpfn} and causal discovery~\cite{robertson2025dopfn}.

\subsection{Generative Imputation}
In tabular imputation, we aim to sample from a conditional distribution of missing entries given observed entries and a missingness mask, under heterogeneous feature types (continuous, discrete, categorical) and strong cross-feature constraints; in many practical settings this conditional mapping is close to deterministic (e.g., structurally constrained columns and rule-like relations). 

Classical deep generative approaches for tabular completion include adversarial and VAE-style imputers such as GAIN~\cite{yoon2018gain} and VAE-based models~\cite{pereira2020vae}, while more recent work leverages diffusion-style denoising for tabular imputation (e.g., Diffputer~\cite{zhang2025diffputer}) and other modern generative formulations~\cite{jolicoeur2024generating}. 

Flow Matching (FM)~\cite{lipman2023flow} provides an alternative ODE/SDE-based route: under the stochastic interpolant view, flows and diffusions can be unified~\cite{albergo2025stochastic}, but FM's inductive bias is often better aligned with deterministic, physics-like constraints because it directly parameterizes a velocity field rather than a denoiser~\cite{kerrigan2023functional}. In addition, FM can learn near-straight transport paths related to optimal transport~\cite{lipman2023flow}, whereas diffusion trajectories are typically more curved; such curvature can amplify numerical integration error~\cite{lee2023minimizing}, making straighter paths particularly attractive for efficient, low-error deterministic imputation in tabular tasks.


\begin{table*}[t]
\centering
\scriptsize
\setlength{\tabcolsep}{4pt}
\renewcommand{\arraystretch}{0.95}
\caption{Imputation benchmark datasets ($N=15$ paper set). $N$: instances; $D$: features.}
\label{tab:imputation-dataset-sizes}
\resizebox{\textwidth}{!}{%
\begin{tabular}{llllllllllllllll}
\toprule
 & bike & blood & climate & compression & credit & diabetes & gesture & ionosphere & iris & obesity & raisin & spam & wine red & wine white & yacht \\
\midrule
$N$ & 8760 & 748 & 540 & 1030 & 690 & 442 & 9522 & 351 & 150 & 2111 & 900 & 4600 & 1599 & 4898 & 308 \\
$D$ & 13 & 5 & 21 & 9 & 16 & 10 & 32 & 35 & 4 & 17 & 8 & 58 & 12 & 12 & 7 \\
\bottomrule
\end{tabular}%
}
\end{table*}


\begin{table}[t]
\centering
\scriptsize
\setlength{\tabcolsep}{3pt}
\renewcommand{\arraystretch}{0.95}
\caption{Dataset summary for classification benchmarks. $N$: instances; $D$: features; Missing \#: missing entries; Missing \%: missing rate. low-missingness group (missing rate $\le 10\%$).}
\label{tab:dataset-summary-0-10}
\resizebox{0.48\textwidth}{!}{%
\begin{tabular}{lrrrrr}
\toprule
Dataset & ID & N & D & Missing \# & Missing \% \\
\midrule
Breast Cancer & 13 & 286 & 9 & 9 & 0.35\% \\
Breast W & 15 & 699 & 10 & 16 & 0.23\% \\
MUSHROOM & 24 & 8124 & 23 & 2480 & 1.33\% \\
Credit Approval & 29 & 690 & 16 & 67 & 0.61\% \\
Dermatology & 35 & 366 & 35 & 8 & 0.06\% \\
Sick & 38 & 3772 & 30 & 6064 & 5.36\% \\
Heart C & 49 & 303 & 14 & 7 & 0.17\% \\
Hepatitis & 55 & 155 & 19 & 167 & 5.67\% \\
Vote & 56 & 435 & 16 & 392 & 5.63\% \\
Lung Cancer & 163 & 32 & 57 & 5 & 0.27\% \\
Eucalyptus & 188 & 736 & 20 & 448 & 3.04\% \\
Cleveland & 194 & 303 & 14 & 6 & 0.14\% \\
Irish & 451 & 500 & 6 & 32 & 1.07\% \\
Analcatdata Broadwaymult & 452 & 285 & 7 & 27 & 1.35\% \\
Biomed & 481 & 209 & 8 & 15 & 0.90\% \\
Pharynx & 738 & 195 & 10 & 2 & 0.10\% \\
Cleveland & 786 & 303 & 13 & 6 & 0.15\% \\
Cholesterol & 798 & 303 & 13 & 6 & 0.15\% \\
Pbcseq & 802 & 1945 & 18 & 1133 & 3.24\% \\
AutoMpg & 831 & 398 & 7 & 6 & 0.22\% \\
Kdd El Nino Small & 839 & 782 & 8 & 466 & 7.45\% \\
AutoHorse & 840 & 205 & 25 & 57 & 1.11\% \\
BreastTumor & 844 & 286 & 9 & 9 & 0.35\% \\
Analcatdata Gsssexsurvey & 852 & 159 & 9 & 6 & 0.42\% \\
Chscase Whale & 939 & 228 & 8 & 20 & 1.10\% \\
Analcatdata Halloffame & 966 & 1340 & 16 & 20 & 0.09\% \\
Analcatdata Birthday & 968 & 365 & 3 & 30 & 2.74\% \\
Analcatdata Draft & 984 & 366 & 4 & 1 & 0.07\% \\
Cylinder Bands & 6332 & 540 & 40 & 999 & 4.62\% \\
SpeedDating & 40536 & 8378 & 121 & 18372 & 1.81\% \\
MiceProtein & 40966 & 1080 & 82 & 1396 & 1.58\% \\
DiabeticMellitus & 41430 & 281 & 97 & 2 & 0.01\% \\
Regime Alimentaire & 42172 & 202 & 19 & 17 & 0.44\% \\
\bottomrule
\end{tabular}%
}
\end{table}

\begin{table}[t]
\centering
\scriptsize
\setlength{\tabcolsep}{3pt}
\renewcommand{\arraystretch}{0.95}
\caption{Dataset summary for classification benchmarks. $N$: instances; $D$: features; Missing \#: missing entries; Missing \%: missing rate. high-missingness group (missing rate $> 10\%$).}
\label{tab:dataset-summary-10plus}
\resizebox{0.48\textwidth}{!}{%
\begin{tabular}{lrrrrr}
\toprule
Dataset & ID & N & D & Missing \# & Missing \% \\
\midrule
Labor & 4 & 57 & 17 & 326 & 33.64\% \\
Colic & 25 & 368 & 26 & 1927 & 20.14\% \\
Colic & 27 & 368 & 23 & 1927 & 22.77\% \\
Heart H & 51 & 294 & 14 & 782 & 19.00\% \\
Analcatdata Reviewer & 460 & 379 & 7 & 1277 & 48.13\% \\
Pbc & 524 & 418 & 20 & 1033 & 12.36\% \\
Pbc & 810 & 418 & 18 & 1239 & 16.47\% \\
Colleges Usnews & 930 & 1302 & 33 & 7830 & 18.22\% \\
Dresses Sales & 23381 & 500 & 13 & 835 & 12.85\% \\
Titanic & 42638 & 891 & 7 & 689 & 11.05\% \\
\bottomrule
\end{tabular}%
}
\end{table}

\section{Implementation Details}\label{sec:implementation}

\paragraph{Implementation of the (Second-Order) SCM prior (abstract view).}\label{sec:impl:scm}
We model data generation using a (second-order) structural causal model (SCM). An SCM is a tuple
\(\mathcal{M}=\langle \mathbf{X},\mathbf{U},\mathcal{F},P(\mathbf{U})\rangle\),
where $\mathbf{X}=(X_1,\dots,X_d)$ are endogenous variables, $\mathbf{U}=(U_1,\dots,U_d)$ are exogenous variables, and
\(\mathcal{F}=\{f_i\}_{i=1}^d\) are structural assignments.
In a \emph{second-order} SCM, we additionally randomize the causal graph $\mathcal{G}$ (equivalently parent sets $\mathrm{PA}(i)$) and the function parameters $\theta=\{\theta_i\}$, inducing a prior over mechanisms and structures.

\paragraph{Nodes (structural assignments).}\label{sec:impl:scm:nodes}
Each node $X_i$ is generated by a structural equation
\[
X_i \leftarrow f_i\!\left(X_{\mathrm{PA}(i)}, U_i;\theta_i\right),
\]
where $X_{\mathrm{PA}(i)}$ denotes the vector of parent variables.
We instantiate $f_i$ as a depth-$L$ feed-forward map (MLP) with elementwise nonlinearity $\phi$:
\begin{equation*}
\begin{aligned}
& h_i^{(0)}=\Pi_i\!\left(X_{\mathrm{PA}(i)}\right),\\
& h_i^{(\ell)}=\phi\!\left(W_i^{(\ell)}h_i^{(\ell-1)}+b_i^{(\ell)}\right)\ \ (\ell=1,\dots,L),\\
& X_i=g_i\!\left(h_i^{(L)}\right)+\sigma_i\,\varepsilon_i,
\end{aligned}
\end{equation*}
where $\Pi_i$ selects (and orders) the parent coordinates, and $\varepsilon_i$ is a standardized noise term.
This parameterization allows nonlinear, compositional effects along directed paths while keeping the SCM semantics explicit.

\paragraph{Edges (graph structure inside the parameterization).}\label{sec:impl:scm:edges}
The directed edge $X_j\to X_i$ is represented by whether $X_j$ is included in $\mathrm{PA}(i)$ and, operationally, by whether the corresponding input-to-hidden connections in $\Pi_i$ (or the first linear map) are active.
Sampling a random graph prior can be implemented by randomly selecting parent sets and/or randomly sparsifying input connections (e.g., via independent masking of weights), which induces a distribution over adjacency patterns while keeping the functional form fixed.
Nonlinearity $\phi$ governs how parent influences compose and interact, i.e., how edges combine beyond additive linear effects.

\paragraph{Exogenous variables (stochasticity and prior variability).}\label{sec:impl:scm:exogenous}
Exogenous variables $U_i$ are sampled i.i.d.\ from a simple base distribution (e.g., $U_i\sim\mathcal{N}(0,1)$), and enter the model either (i) as explicit inputs to $f_i$ or (ii) through an additive output noise term $\sigma_i\varepsilon_i$.
Randomizing $(\mathcal{G},\theta)$ further yields sample-to-sample mechanism variability, which is crucial for representing a rich prior over tabular-generating processes.

\paragraph{Concrete prior instantiation (as used in our code).}\label{sec:impl:scm:prior-concrete}
To make $P(\mathcal{G},\theta)$ reproducible, we use a fixed generator family where each mechanism is a depth-$L$ MLP with hidden width $H$ and additive Gaussian noise. Mechanism diversity (and effective sparsity) is induced by randomly masking weights, optional block-wise sparsification, and random rotations/selection of observed coordinates from a shared latent representation; for causal tasks we enforce $H\ge d_y+2d$ so that features, labels, and missingness scores can be sampled from the same representation. Exogenous causes are Gaussian, optionally with per-cause mean/scale randomization.

\paragraph{Prior over missingness mechanisms (randomized deterministic gates).}\label{sec:impl:scm:missingness-prior}
We generate masks via a randomized \emph{score-and-quantile} gate driven by a latent representation $z$ (a concatenation of hidden states from the SCM generator). A score network $g_\psi$ maps $z$ to per-feature scores; to induce heterogeneity, we draw a layer index $\ell_j$ uniformly for each feature $j$ and use the corresponding score $s_{ij}$. We then draw a per-feature quantile level $\alpha_j\sim\mathrm{Unif}[\alpha_{\min},\alpha_{\max}(t)]$, set the threshold $\tau_j$ to the empirical $\alpha_j$-quantile of $\{s_{ij}\}_i$, and define $M_{ij}=1[s_{ij}\le \tau_j]$; label dimensions are always forced observed. In our default configuration, $\alpha_{\min}=0.0$ and $\alpha_{\max}(t)$ is linearly warmed up from $0.3$ to $0.8$ over the first $1000$ optimization steps.

\paragraph{Concrete pre-training details (for reproducibility).}\label{sec:impl:pretraining}
\textbf{SCM prior (graph and mechanisms).}
Each synthetic task samples a nonlinear data-generating mechanism realized by a feed-forward network of depth $L\ge 2$, hidden width $H$, and an activation chosen from \{tanh, identity, ReLU\}. Mechanism diversity and effective graph sparsity are induced by (i) Bernoulli masking of weight matrices (with no masking in the first layer), (ii) optional block-wise sparsification, and (iii) random feature rotations and random selection/ordering of observed coordinates from a shared latent representation. Exogenous causes are sampled from a standard Gaussian, with optional per-cause mean/scale randomization; output noise is Gaussian with log-uniformly sampled scale. For classification tasks, we sample the number of classes uniformly from 2 to 10 and use a categorical-feature probability of $0.2$.

\textbf{Task size.}
We pre-train on tasks with feature dimension up to 100 and $n=1152$ samples per task (with evaluation positions at $0.95n$).

\textbf{Missingness logic.}
We use the score-and-quantile gate above, with a convolutional score network of two layers (kernel size 7) and tanh nonlinearity; missingness rates are controlled by $\alpha_{\min}=0.0$ and a warmed-up $\alpha_{\max}(t)$ (from $0.3$ to $0.8$ over 1000 steps).

\textbf{Model architectures.}
The PFN is a Transformer encoder with model width 512, 12 layers, 4 attention heads, feed-forward width 1024, dropout 0, and GELU activations; we embed both feature values and the binary mask and add them to form token embeddings. The conditional flow-matching head is a lightweight 2-layer Transformer that predicts the velocity field for masked entries using a sinusoidal time embedding; during training, the loss is computed only on masked entries, while the state is noisy everywhere.

\textbf{Optimization.}
We train with Adam at learning rate $3\times 10^{-5}$, cosine schedule with 20 warmup epochs and minimum learning rate $10^{-8}$, no weight decay, and mixed precision; we also use a self-distillation auxiliary loss with weight 0.1. Synthetic tasks are generated on-the-fly from the prior; the default batch size is 64 tasks. For reference, we pre-train \textsc{PFN-Flow} on a single NVIDIA RTX PRO 6000 GPU for 120 hours.

\subsection{Missing Data Classification Experiments}\label{sec:impl:cls}

\paragraph{Baselines.}\label{sec:impl:cls:baselines}
We compare strong, widely competitive tabular predictors, \textsc{XGBoost}~\cite{DBLP:conf/kdd/ChenG16}, \textsc{LightGBM}~\cite{DBLP:conf/nips/KeMFWCMYL17}, and \textsc{CatBoost}~\cite{DBLP:conf/nips/ProkhorenkovaGV18}. We choose these baselines not merely because they are ``standard'', but because they are consistently high-performing on real-world tabular benchmarks and remain strong when combined with mature missing-data pipelines. Concretely, we evaluate missing-data handling strategies spanning plug-in imputers such as MissForest~\cite{Stekhoven2011MissForestN}, AutoML/statistical imputers such as HyperImpute~\cite{jarrett2022hyperimpute}, modern learned imputers including \textsc{Remasker}~\cite{du2024remasker}, \textsc{GRAPE}~\cite{you2020handling}, \textsc{IGRM}~\cite{zhong2023data}, \textsc{MCFlow}~\cite{richardson2020mcflow}, and Transformed Distribution Matching (TDM)~\cite{DBLP:conf/icml/0001SDB23}, as well as OT-family methods including \textsc{OT}~\cite{muzellec2020missing}, \textsc{OTM}~\cite{ot4sl}, and masked OT variants (MOT). For MOT, we report several practical variants that share the same OT-style objective but differ in the conditional model class used in the refinement step (e.g., linear vs.\ MLP), with a Sinkhorn-style OT baseline as a reference.

\paragraph{Dataset selection and benchmark name.}\label{sec:impl:cls:datasets}
To construct a reliable benchmark with real missingness, we filter OpenML tabular classification datasets using two simple rules:
(i) the empirical missing rate is strictly positive (missing rate $>0$), and
(ii) the dataset falls within the practical operating range of TabPFN (e.g., feature dimensionality and sample size constraints for feasible inference).
We call the resulting benchmark the \emph{OpenML Missingness Benchmark (OMB)}.

We summarize the classification datasets in Tables~\ref{tab:dataset-summary-0-10} and~\ref{tab:dataset-summary-10plus}, which are selected from OpenML and grouped by empirical missing rate, including dataset size ($N$), inferred feature dimension ($D$), and missingness statistics (total missing entries and overall missing rate).

\paragraph{Full method tables (all methods).}\label{sec:impl:cls:runtime}
For the main paper (Sec.~\ref{sec:exp:cls-auc}), we report a compact Top-15 summary for each missingness group in
Tables~\ref{tab:auc-method-summary-0-10} and~\ref{tab:auc-method-summary-10plus}.
For completeness, we provide the corresponding full method summaries (all evaluated methods) in
Tables~\ref{tab:auc-method-summary-0-10-all} and~\ref{tab:auc-method-summary-10plus-all} below. We report end-to-end wall-clock time per method as measured by our runtime benchmark. Concretely, for each method we measure the time to obtain predictions on a fixed $n=1000, d=50$ input, including fitting the imputer (if any) on the training split and applying it to train/test, plus fitting the classifier and running inference on the test split; PFN-family predictors require no target-dataset gradient training and are timed as amortized forward-pass inference.

\subsection{Missing Data Imputation Experiments}\label{sec:impl:imp}

\paragraph{Evaluation protocol.}\label{sec:impl:imp:protocol}
All methods are evaluated under exactly the same splits and masks through the unified baseline runners; for robustness, we report results on the intersection of datasets shared by all methods (here $N=15$).
We summarize the imputation benchmark datasets in Table~\ref{tab:imputation-dataset-sizes} (dataset size $N$ and feature dimension $D$).

\paragraph{Scalability to Larger Datasets.}
Empirically, the original \textsc{TabPFN} (Nature) recipe tends to dominate mainly in the small-data regime (roughly $N<3{,}000$ instances), while follow-up lines such as \textsc{TabICL} and \textsc{TabPFN-v2.5} introduce additional training/inference techniques that improve performance beyond that range and often dominate on larger datasets.
Since our work does not adopt those large-$N$ engineering tricks, we focus our imputation evaluation on real-world datasets with $N<10{,}000$, which is sufficient for validating the missingness-inference component and the end-to-end uncertainty propagation studied in this paper.

\paragraph{Baselines.}\label{sec:impl:imp:baselines}
We compare against representative imputation baselines spanning masked autoencoding, graph-based reconstruction, generative flow models, and OT-based methods: \textsc{Remasker}~\cite{du2024remasker}, \textsc{GRAPE}~\cite{you2020handling}, \textsc{IGRM}~\cite{zhong2023data}, \textsc{MCFlow}~\cite{richardson2020mcflow}, \textsc{HyperImpute}~\cite{jarrett2022hyperimpute}, \textsc{MOT}~\cite{muzellec2020missing}, and \textsc{OTM}~\cite{ot4sl}.

\paragraph{MNAR\_logistic\_M2M (logistic MNAR with $M\!\to\!M$ propagation).}\label{sec:impl:imp:mnar}
Let $X\in\mathbb{R}^{n\times d}$ be the (fully observed) feature matrix before masking.
We sample a subset of columns $S$ as \emph{logistic inputs} with $|S|=d_{\text{in}}=\max(\lfloor qd\rfloor,1)$ (with $q=0.3$ when $p\le 0.3$ and $q=0.1$ otherwise), and set $T=\{1,\dots,d\}\setminus S$.
We first mask the input columns $S$ independently with probability $p$ (MCAR on $S$), yielding an input mask $M_{i,S}$.
Then for each $j\in T$, we draw the target mask using a logistic model that depends on both $X$ and the already-sampled missingness in $S$:
\begin{equation}
    \begin{aligned}
    &P(M_{ij}=1\mid X, M_{i,S}) = \\
    &\sigma\!\Big(\big(X_{i,S}\odot M_{i,S}\big)^\top w_j+\big(\mathbf{1}-M_{i,S}\big)^\top v_j+b_j\Big),    
    \end{aligned}
\end{equation}
where $v_j$ linearly transforms the missingness indicators (thus capturing $M\!\to\!M$ propagation), and $b_j$ is chosen (via bisection) to match the target missing rate $p$.

\begin{table}[H]
\centering
\scriptsize
\setlength{\tabcolsep}{3pt}
\renewcommand{\arraystretch}{0.95}
\caption{Method-level \textbf{AUC} summary: mean/std, runtime, and average ranks. Missing rate $\le 10\%$. AUC is reported as mean$\pm$std across datasets. \textbf{Rank} denotes the average per-dataset rank. \textbf{Time} is reported as mean$\pm$std(seconds) for time cost per 1000 instance with 50 variables. Datasets: 33. Showing all methods.}
\label{tab:auc-method-summary-0-10-all}
\resizebox{0.48\textwidth}{!}{%
\begin{tabular}{lrrr}
\toprule
\multicolumn{1}{c}{Method} & \multicolumn{1}{c}{AUC} & \multicolumn{1}{c}{Rank} & \multicolumn{1}{c}{Time (s)} \\
\midrule
PFN-Flow & $0.8694\pm0.1448$ & 16.83 & $0.061\pm0.003$ \\
CatBoost + TDM & $0.8664\pm0.1512$ & 20.77 & $160.830\pm8.329$ \\
PFN-NSM & $0.8663\pm0.1499$ & 18.70 & $0.061\pm0.003$ \\
CatBoost + GRAPE & $0.8660\pm0.1521$ & 21.65 & $43.530\pm0.294$ \\
CatBoost + IGRM & $0.8659\pm0.1513$ & 21.44 & $44.152\pm0.241$ \\
CatBoost + OTM-SK & $0.8657\pm0.1514$ & 20.70 & $82.744\pm2.056$ \\
CatBoost + OTM-RR & $0.8656\pm0.1504$ & 22.02 & $42.211\pm0.211$ \\
CatBoost + Remasker & $0.8656\pm0.1501$ & 21.32 & $551.288\pm9.349$ \\
CatBoost + MOT-MLP & $0.8646\pm0.1527$ & 22.94 & $165.097\pm0.798$ \\
CatBoost + MCFlow & $0.8646\pm0.1522$ & 20.91 & $65.533\pm0.612$ \\
CatBoost + HyperImpute & $0.8645\pm0.1522$ & 23.11 & $207.728\pm0.543$ \\
CatBoost (Raw) & $0.8642\pm0.1523$ & 22.17 & $37.031\pm0.186$ \\
CatBoost + MissForest & $0.8641\pm0.1526$ & 23.14 & $63.092\pm1.692$ \\
CatBoost + MOT-LIN & $0.8639\pm0.1527$ & 24.15 & $164.264\pm3.783$ \\
XGBoost + GRAPE & $0.8638\pm0.1530$ & 24.29 & $11.187\pm0.229$ \\
XGBoost + IGRM & $0.8636\pm0.1531$ & 23.48 & $11.809\pm0.156$ \\
XGBoost + MCFlow & $0.8634\pm0.1523$ & 22.80 & $33.190\pm0.584$ \\
XGBoost (Raw) & $0.8634\pm0.1514$ & 23.05 & $4.689\pm0.022$ \\
XGBoost + TDM & $0.8633\pm0.1524$ & 24.98 & $128.488\pm8.327$ \\
XGBoost + Remasker & $0.8631\pm0.1513$ & 24.92 & $518.946\pm9.348$ \\
XGBoost + OTM-SK & $0.8627\pm0.1523$ & 23.77 & $50.401\pm2.048$ \\
XGBoost + HyperImpute & $0.8626\pm0.1523$ & 25.12 & $175.385\pm0.510$ \\
XGBoost + MOT-MLP & $0.8622\pm0.1532$ & 25.17 & $132.754\pm0.776$ \\
XGBoost + OTM-RR & $0.8621\pm0.1522$ & 27.15 & $9.869\pm0.103$ \\
XGBoost + MissForest & $0.8614\pm0.1534$ & 28.29 & $30.749\pm1.682$ \\
XGBoost + MOT-LIN & $0.8611\pm0.1536$ & 26.62 & $131.921\pm3.778$ \\
LightGBM + GRAPE & $0.8576\pm0.1585$ & 25.88 & $14.986\pm0.240$ \\
TabPFN + MCFlow & $0.8574\pm0.1592$ & 26.33 & $28.563\pm0.583$ \\
LightGBM + TDM & $0.8574\pm0.1585$ & 27.44 & $132.287\pm8.328$ \\
TabPFN + Remasker & $0.8574\pm0.1599$ & 26.67 & $514.319\pm9.348$ \\
LightGBM + IGRM & $0.8573\pm0.1584$ & 26.47 & $15.608\pm0.172$ \\
TabPFN + HyperImpute & $0.8571\pm0.1600$ & 26.94 & $170.758\pm0.510$ \\
TabPFN + MOT-MLP & $0.8571\pm0.1590$ & 26.33 & $128.127\pm0.776$ \\
TabPFN + TDM & $0.8570\pm0.1597$ & 27.64 & $123.860\pm8.327$ \\
TabPFN + MissForest & $0.8561\pm0.1613$ & 27.64 & $26.122\pm1.681$ \\
LightGBM + Remasker & $0.8559\pm0.1585$ & 29.12 & $522.745\pm9.348$ \\
LightGBM + MCFlow & $0.8557\pm0.1601$ & 27.17 & $36.990\pm0.588$ \\
LightGBM + HyperImpute & $0.8556\pm0.1589$ & 29.09 & $179.184\pm0.515$ \\
LightGBM + OTM-RR & $0.8556\pm0.1586$ & 29.91 & $13.668\pm0.126$ \\
LightGBM + OTM-SK & $0.8555\pm0.1596$ & 28.27 & $54.200\pm2.049$ \\
TabPFN + OTM-SK & $0.8554\pm0.1616$ & 28.74 & $45.774\pm2.048$ \\
TabPFN + GRAPE & $0.8547\pm0.1598$ & 27.38 & $6.560\pm0.228$ \\
LightGBM + MOT-MLP & $0.8545\pm0.1602$ & 29.44 & $136.553\pm0.779$ \\
TabPFN (Raw) & $0.8545\pm0.1610$ & 26.73 & $0.061\pm0.003$ \\
LightGBM + MissForest & $0.8541\pm0.1596$ & 32.41 & $34.548\pm1.683$ \\
LightGBM (Raw) & $0.8537\pm0.1626$ & 28.74 & $8.488\pm0.075$ \\
LightGBM + MOT-LIN & $0.8535\pm0.1609$ & 30.79 & $135.721\pm3.779$ \\
TabPFN + OTM-RR & $0.8532\pm0.1602$ & 29.44 & $5.241\pm0.101$ \\
TabPFN + IGRM & $0.8530\pm0.1593$ & 27.71 & $7.182\pm0.154$ \\
TabPFN + MOT-LIN & $0.8530\pm0.1580$ & 29.24 & $127.294\pm3.778$ \\
\bottomrule
\end{tabular}%
}
\end{table}

\paragraph{Metric and aggregation (MAE, AVG, Avg Rank).}\label{sec:impl:imp:metrics}
We report \emph{out-of-sample MAE} on the masked entries of the test split (lower is better).
For each dataset/method, we compute mean$\pm$std across the $K$ MNAR mask realizations.
To summarize performance across datasets, we report:
(i) \textbf{AVG}: for each mask, we average MAE across the $N$ datasets and then compute mean$\pm$std over the $K$ masks; and
(ii) \textbf{Avg Rank}: within each dataset, methods are ranked by their mask-averaged MAE (lower is better; ties receive average rank), and we average ranks across datasets.
Avg Rank complements absolute MAE by emphasizing cross-dataset consistency under heterogeneous feature distributions and MNAR perturbations.

\begin{table}[H]
\centering
\scriptsize
\setlength{\tabcolsep}{3pt}
\renewcommand{\arraystretch}{0.95}
\caption{Missing rate $> 10\%$. AUC is reported as mean$\pm$std across datasets. \textbf{Rank} denotes the average per-dataset rank. \textbf{Time} is reported as mean$\pm$std(seconds) for time cost per 1000 instance with 50 variables. Datasets: 10. Showing all methods.}
\label{tab:auc-method-summary-10plus-all}
\resizebox{0.48\textwidth}{!}{%
\begin{tabular}{lrrr}
\toprule
\multicolumn{1}{c}{Method} & \multicolumn{1}{c}{AUC} & \multicolumn{1}{c}{Rank} & \multicolumn{1}{c}{Time (s)} \\
\midrule
PFN-Flow & $0.8172\pm0.1099$ & 11.40 & $0.061\pm0.003$ \\
XGBoost + MOT-LIN & $0.8111\pm0.1217$ & 12.00 & $131.921\pm3.778$ \\
XGBoost + MOT-MLP & $0.8073\pm0.1216$ & 17.60 & $132.754\pm0.776$ \\
PFN-NSM & $0.8065\pm0.1219$ & 16.20 & $0.061\pm0.003$ \\
XGBoost + IGRM & $0.8048\pm0.1127$ & 17.90 & $11.809\pm0.156$ \\
XGBoost (Raw) & $0.8047\pm0.1117$ & 15.50 & $4.689\pm0.022$ \\
CatBoost + MOT-LIN & $0.8042\pm0.1228$ & 22.40 & $164.264\pm3.783$ \\
CatBoost + MOT-MLP & $0.8039\pm0.1249$ & 20.30 & $165.097\pm0.798$ \\
CatBoost (Raw) & $0.8034\pm0.1213$ & 23.70 & $37.031\pm0.186$ \\
XGBoost + GRAPE & $0.8032\pm0.1109$ & 18.50 & $11.187\pm0.229$ \\
XGBoost + OTM-RR & $0.8028\pm0.1206$ & 21.40 & $9.869\pm0.103$ \\
TabPFN + GRAPE & $0.8016\pm0.1301$ & 24.90 & $6.560\pm0.228$ \\
XGBoost + Remasker & $0.8014\pm0.1202$ & 25.70 & $518.946\pm9.348$ \\
TabPFN + MOT-LIN & $0.8011\pm0.1290$ & 23.30 & $127.294\pm3.778$ \\
TabPFN + MOT-MLP & $0.8006\pm0.1311$ & 23.50 & $128.127\pm0.776$ \\
TabPFN + TDM & $0.8006\pm0.1294$ & 24.80 & $123.860\pm8.327$ \\
CatBoost + Remasker & $0.8005\pm0.1298$ & 27.80 & $551.288\pm9.349$ \\
TabPFN + IGRM & $0.8004\pm0.1283$ & 24.90 & $7.182\pm0.154$ \\
TabPFN + MissForest & $0.8002\pm0.1272$ & 23.20 & $26.122\pm1.681$ \\
TabPFN + HyperImpute & $0.8001\pm0.1285$ & 23.10 & $170.758\pm0.510$ \\
XGBoost + OTM-SK & $0.8000\pm0.1118$ & 26.90 & $50.401\pm2.048$ \\
CatBoost + OTM-SK & $0.8000\pm0.1180$ & 26.10 & $82.744\pm2.056$ \\
TabPFN + Remasker & $0.7998\pm0.1299$ & 26.40 & $514.319\pm9.348$ \\
CatBoost + GRAPE & $0.7997\pm0.1188$ & 26.50 & $43.530\pm0.294$ \\
XGBoost + MissForest & $0.7992\pm0.1190$ & 22.10 & $30.749\pm1.682$ \\
CatBoost + IGRM & $0.7986\pm0.1155$ & 25.40 & $44.152\pm0.241$ \\
TabPFN (Raw) & $0.7985\pm0.1269$ & 23.50 & $0.061\pm0.003$ \\
TabPFN + MCFlow & $0.7983\pm0.1214$ & 25.80 & $28.563\pm0.583$ \\
XGBoost + HyperImpute & $0.7982\pm0.1281$ & 23.50 & $175.385\pm0.510$ \\
TabPFN + OTM-RR & $0.7972\pm0.1272$ & 28.20 & $5.241\pm0.101$ \\
CatBoost + TDM & $0.7970\pm0.1187$ & 28.90 & $160.830\pm8.329$ \\
XGBoost + TDM & $0.7970\pm0.1141$ & 26.40 & $128.488\pm8.327$ \\
CatBoost + HyperImpute & $0.7964\pm0.1331$ & 28.20 & $207.728\pm0.543$ \\
TabPFN + OTM-SK & $0.7955\pm0.1315$ & 28.90 & $45.774\pm2.048$ \\
CatBoost + MissForest & $0.7947\pm0.1236$ & 29.40 & $63.092\pm1.692$ \\
CatBoost + OTM-RR & $0.7945\pm0.1253$ & 31.70 & $42.211\pm0.211$ \\
XGBoost + MCFlow & $0.7935\pm0.1187$ & 21.40 & $33.190\pm0.584$ \\
CatBoost + MCFlow & $0.7901\pm0.1161$ & 25.20 & $65.533\pm0.612$ \\
LightGBM + MOT-LIN & $0.7574\pm0.1384$ & 26.85 & $135.721\pm3.779$ \\
LightGBM (Raw) & $0.7565\pm0.1364$ & 26.25 & $8.488\pm0.075$ \\
LightGBM + GRAPE & $0.7561\pm0.1351$ & 27.65 & $14.986\pm0.240$ \\
LightGBM + IGRM & $0.7549\pm0.1345$ & 27.75 & $15.608\pm0.172$ \\
LightGBM + MOT-MLP & $0.7537\pm0.1363$ & 30.55 & $136.553\pm0.779$ \\
LightGBM + OTM-RR & $0.7523\pm0.1374$ & 33.35 & $13.668\pm0.126$ \\
LightGBM + MCFlow & $0.7509\pm0.1409$ & 29.35 & $36.990\pm0.588$ \\
LightGBM + MissForest & $0.7505\pm0.1337$ & 33.65 & $34.548\pm1.683$ \\
LightGBM + TDM & $0.7480\pm0.1336$ & 36.85 & $132.287\pm8.328$ \\
LightGBM + Remasker & $0.7477\pm0.1338$ & 37.55 & $522.745\pm9.348$ \\
LightGBM + HyperImpute & $0.7477\pm0.1383$ & 34.35 & $179.184\pm0.515$ \\
LightGBM + OTM-SK & $0.7472\pm0.1297$ & 38.25 & $54.200\pm2.049$ \\
\bottomrule
\end{tabular}%
}
\end{table}

\paragraph{Protocol (data split + repeated MNAR masks).}\label{sec:impl:imp:mnar-protocol}
We follow the DiffPuter preprocessing pipeline: each dataset is split into train/test with a 70/30 ratio.
On the test split, we generate $K=10$ independently sampled MNAR masks with missing rate $p=0.3$, using the MNAR logistic mechanism \texttt{MNAR\_logistic\_M2M}.

\section{Theoretical Analysis}\label{app:theoretical-analysis}

\subsection{The Random Variables are Well-Defined}\label{pf:defxm}

\textbf{Theorem:} The objects $X_m$ and $X_m^c$, defined as projections dependent on a random mask $M$, are well-defined random variables (measurable functions).

\begin{proof}
    Let $(\Omega, \mathcal{F}, P)$ be the underlying probability space. Let $\mathbf{X}: \Omega \to \mathbb{R}^d$ and $M: \Omega \to \{0, 1\}^d$ be random variables. Let $S = \{0, 1\}^d$ denote the finite set of all possible masks.
    
    To establish that $X_m$ is a well-defined random variable, we must define a measurable space for its codomain. Since the dimension of $X_m$ depends on $M$, the codomain is the disjoint union space $\mathbb{V} = \bigsqcup_{m \in S} \mathbb{R}^{|m|}$, equipped with the $\sigma$-algebra generated by the Borel sets on each component. For any measurable set $B \subseteq \mathbb{V}$, the preimage is:

    \begin{equation}
    \begin{aligned}
    X_m^{-1}(B) = \bigcup_{m \in S} & \{ \omega \in \Omega \mid M(\omega) = m \} \\
    & \cap \{ \omega \in \Omega \mid \text{Proj}(m, \mathbf{X}(\omega)) \in B \cap \mathbb{R}^{|m|} \}.
    \end{aligned}
    \end{equation}
    For a fixed $m$, the projection $\text{Proj}(m, \cdot)$ is a continuous linear map, making $\text{Proj}(m, \mathbf{X})$ a random variable. Thus, the second set in the intersection is in $\mathcal{F}$. Since $M$ is a random variable, $\{M=m\} \in \mathcal{F}$. As $\mathcal{F}$ is closed under finite intersections and unions, $X_m^{-1}(B) \in \mathcal{F}$, proving $X_m$ is measurable.
    
    Similarly, $X_m^c$ maps to the product space $\mathbb{V}^c \times S$, where $\mathbb{V}^c = \bigsqcup_{m \in S} \mathbb{R}^{d-|m|}$. For any measurable set $C \times D \subseteq \mathbb{V}^c \times S$, the preimage decomposes as:
    \begin{equation}
    \begin{aligned}
    &(X_m^c)^{-1}(C \times D) = \\
    & \bigcup_{m \in D} \{ \omega \in \Omega \mid M(\omega) = m \} \\
    & \quad \quad  \cap \{ \omega \in \Omega \mid \text{Proj}(m^c, \mathbf{X}(\omega)) \in C \cap \mathbb{R}^{d-|m|} \}.
    \end{aligned}
    \end{equation}
    By the same logic, each term in the union is measurable. Therefore, $X_m^c$ is a well-defined random variable.
    \end{proof}

\subsection{Proof of Theorem~\ref{thm:pfn-risk}}\label{pf:pfn-risk}
\begin{proof}
        Fix $(\mathcal{M},D)$ and $x$. For any $R(\cdot)$ on labels,
        \begin{equation*}
            \begin{aligned}
            & \mathbb{E}_{y\sim P^\star(\cdot\mid x,D)}[-\log R(y)] \\
            & =
            H\big(P^\star(\cdot\mid x,D)\big) 
            +
            \mathrm{KL}\big(P^\star(\cdot\mid x,D)\,\|\,R(\cdot)\big),
            \end{aligned}
        \end{equation*}
        where $H(\cdot)$ is Shannon entropy. Taking $R(\cdot)=P_\phi(\cdot\mid x,D)$ and averaging over $(x,y,D)$ under $\Pi$ yields
        \begin{equation*}
            \begin{aligned}
       & \mathcal{R}(\phi) \\
        =
       & \mathbb{E}_{(\mathcal{M},D)\sim\Pi}\,\mathbb{E}_{x\sim P^\star(\cdot\mid D)}
        \Big[H\big(P^\star(\cdot\mid x,D)\big)\Big]
        + \\
       & \mathbb{E}_{(\mathcal{M},D)\sim\Pi}\,\mathbb{E}_{x\sim P^\star(\cdot\mid D)}
        \Big[\mathrm{KL}\big(P^\star(\cdot\mid x,D)\,\|\,P_\phi(\cdot\mid x,D)\big)\Big].
        \end{aligned}
        \end{equation*}
        The first term is independent of $\phi$, hence any minimizer of $\mathcal{R}(\phi)$ also minimizes the expected conditional KL, proving (i).
        If the minimum expected KL is zero (realizability), then any minimizer must satisfy $\mathrm{KL}(P^\star\|P_{\phi^\star})=0$ $\Pi$-a.s., hence $P_{\phi^\star}=P^\star$ $\Pi$-a.s., proving (ii). Statement (iii) is immediate from (i).
\end{proof}

\subsection{Assumptions and Proof of Theorem~\ref{thm:post-int}}\label{pf:post-int}

\begin{assumption}[Posterior approximation and predictor regularity]\label{assump:post-int}
    We assume:
    \begin{enumerate}
        \item[(A0)] $\mu$ and $\nu$ have finite first moments under $d$ (so $W_1$ is well-defined).
        \item[(A1)] The conditional complete-data target
        $
        h^\star(x_m) := \mathbb{E}\!\left[g(Y)\mid X_m=x_m, X_m^c, D_m^c\right]
        $
        is $L$-Lipschitz in $x_m$ under $d$.
        \item[(A2)] The posterior approximation satisfies
        $
        W_1(\mu,\nu) \le \varepsilon_{\text{post}}.
        $
        \item[(A3)] A predictor $\hat h_\phi$ (e.g., PFN) satisfies an integrated error bound under $\nu$:
        $
        \left|\mathbb{E}_{X_m\sim \nu}\!\left[\hat h_\phi(X_m)-h^\star(X_m)\right]\right|
        \le \varepsilon_{\text{pred}}.
        $
    \end{enumerate}
\end{assumption}

We provide a fully detailed proof of Theorem~\ref{thm:post-int}. Throughout this section, we fix $(X_m^c, D_m^c)$ and use the shorthand
$\mu := P^\star(\cdot \mid X_m^c, D_m^c)$ and $\nu := Q_\theta(\cdot \mid X_m^c, D_m^c)$ introduced in Assumption~\ref{assump:post-int}.

\begin{lemma}[Tower property representation]\label{lem:post-int-tower}
Let $h^\star(x_m):=\mathbb{E}\!\left[g(Y)\mid X_m=x_m, X_m^c, D_m^c\right]$. Then
\begin{equation*}
    \begin{aligned}
    T_g^\star(X_m^c, D_m^c)
    &= \mathbb{E}\!\left[g(Y)\mid X_m^c, D_m^c\right] \\
    &= \mathbb{E}_{X_m\sim \mu}\!\left[h^\star(X_m)\right].
    \end{aligned}
\end{equation*}
\end{lemma}
\begin{proof}
By the tower property of conditional expectation,
\begin{equation*}
    \begin{aligned}
    T_g^\star(X_m^c, D_m^c)
    &= \mathbb{E}\!\left[g(Y)\mid X_m^c, D_m^c\right] \\
    &= \mathbb{E}\!\left[\mathbb{E}\!\left[g(Y)\mid X_m, X_m^c, D_m^c\right]\ \middle|\ X_m^c, D_m^c\right].
    \end{aligned}
\end{equation*}

The inner conditional expectation equals $h^\star(X_m)$ by definition, and under the conditioning $(X_m^c, D_m^c)$ the random variable $X_m$ is distributed as $\mu$, hence the result.
\end{proof}

\begin{lemma}[Error decomposition]\label{lem:post-int-decomp}
With $\widehat T_g(X_m^c, D_m^c):=\mathbb{E}_{X_m\sim \nu}[\hat h_\phi(X_m)]$,
\[
\left|T_g^\star-\widehat T_g\right|
\le
\left|\mathbb{E}_{\mu}[h^\star]-\mathbb{E}_{\nu}[h^\star]\right|
\;+\;
\left|\mathbb{E}_{\nu}[h^\star-\hat h_\phi]\right|.
\]
\end{lemma}
\begin{proof}
By Lemma~\ref{lem:post-int-tower}, $T_g^\star=\mathbb{E}_{\mu}[h^\star(X_m)]$. Therefore,

\begin{equation*}
    \begin{aligned}
 \left|T_g^\star-\widehat T_g\right|
& =\left|\mathbb{E}_{\mu}[h^\star(X_m)]-\mathbb{E}_{\nu}[\hat h_\phi(X_m)]\right| \\
& \le
\left|\mathbb{E}_{\mu}[h^\star]-\mathbb{E}_{\nu}[h^\star]\right|
\;+\;
\left|\mathbb{E}_{\nu}[h^\star-\hat h_\phi]\right|,
\end{aligned}
\end{equation*}

by adding and subtracting $\mathbb{E}_{\nu}[h^\star(X_m)]$ and applying the triangle inequality.
\end{proof}

\begin{lemma}[Lipschitz test functions and Wasserstein-1]\label{lem:post-int-kr}
If $h^\star$ is $L$-Lipschitz under $d$ and $W_1$ is induced by $d$, then
\[
\left|\mathbb{E}_{\mu}[h^\star]-\mathbb{E}_{\nu}[h^\star]\right|\le L\,W_1(\mu,\nu).
\]
\end{lemma}
\begin{proof}
Assumption~\ref{assump:post-int}(A0) ensures $W_1(\mu,\nu)$ is well-defined. By Kantorovich--Rubinstein duality,
\[
W_1(\mu,\nu)=\sup_{\mathrm{Lip}(f)\le 1}\left(\mathbb{E}_{\mu}[f]-\mathbb{E}_{\nu}[f]\right).
\]
If $h^\star$ is $L$-Lipschitz, then $f:=h^\star/L$ is 1-Lipschitz when $L>0$ (the claim is trivial if $L=0$). Hence
\begin{equation*}
    \begin{aligned}
\left|\mathbb{E}_{\mu}[h^\star]-\mathbb{E}_{\nu}[h^\star]\right|
&=L\left|\mathbb{E}_{\mu}[f]-\mathbb{E}_{\nu}[f]\right| \\
& \le L\,W_1(\mu,\nu).
\end{aligned}
\end{equation*}
\end{proof}

\begin{proof}[Proof of Theorem~\ref{thm:post-int}]
    By Lemma~\ref{lem:post-int-decomp},
    \[
    \left|T_g^\star-\widehat T_g\right|
    \le
    \left|\mathbb{E}_{\mu}[h^\star]-\mathbb{E}_{\nu}[h^\star]\right|
    \;+\;
    \left|\mathbb{E}_{\nu}[h^\star-\hat h_\phi]\right|.
    \]
    The second term is bounded by Assumption~\ref{assump:post-int}(A3), yielding $\left|\mathbb{E}_{\nu}[h^\star-\hat h_\phi]\right|\le \varepsilon_{\text{pred}}$.
    For the first term, Lemma~\ref{lem:post-int-kr} gives
    $\left|\mathbb{E}_{\mu}[h^\star]-\mathbb{E}_{\nu}[h^\star]\right|\le L\,W_1(\mu,\nu)$,
    and Assumption~\ref{assump:post-int}(A2) implies $W_1(\mu,\nu)\le \varepsilon_{\text{post}}$.
    Combining yields
    \[
    \left|T_g^\star-\widehat T_g\right|\le L\,\varepsilon_{\text{post}}+\varepsilon_{\text{pred}}.
    \]
\end{proof}

\subsection{Proof of Corollary~\ref{cor:jensen-gap}}\label{pf:jensen-gap}
\begin{proof}
Let $\mu:=P^\star(\cdot \mid X_m^c, D_m^c)$ as in the corollary statement. By assumption, $h^\star$ is strictly convex or strictly concave, $\mu$ is non-degenerate, and $\mathbb{E}_{X_m\sim \mu}[X_m]$ exists.

If $h^\star$ is strictly convex, then by the strict Jensen inequality,
\[
\mathbb{E}_{X_m\sim \mu}[h^\star(X_m)]>h^\star(\mathbb{E}_{X_m\sim \mu}[X_m]).
\]
If $h^\star$ is strictly concave, the inequality reverses. In either case, since $\mu$ is non-degenerate, the inequality is strict, yielding
\[
\left|\mathbb{E}_{X_m\sim \mu}[h^\star(X_m)]-h^\star(\mathbb{E}_{X_m\sim \mu}[X_m])\right|>0.
\]
\end{proof}

\subsection{Proof of Corollary~\ref{cor:forced-same-dist-mismatch}}\label{pf:forced-same-dist-mismatch}
\begin{proof}
Recall the variational characterization
\[
\mathrm{TV}(P_0,P_1)=\sup_{\|f\|_\infty\le 1}\frac12\big|\mathbb{E}_{P_0}f-\mathbb{E}_{P_1}f\big|.
\]
Let $\varepsilon>0$ and choose $f_\varepsilon$ with $\|f_\varepsilon\|_\infty\le 1$ such that
\[
\frac12\big|\mathbb{E}_{P_0}f_\varepsilon-\mathbb{E}_{P_1}f_\varepsilon\big|
\ge \mathrm{TV}(P_0,P_1)-\varepsilon.
\]
For any candidate $\widetilde P$, the triangle inequality gives

\begin{equation*}
    \begin{aligned}
& \big|\mathbb{E}_{P_0}f_\varepsilon-\mathbb{E}_{P_1}f_\varepsilon\big| 
 \le
\big|\mathbb{E}_{P_0}f_\varepsilon-\mathbb{E}_{\widetilde P}f_\varepsilon\big| 
 \;+\;
\big|\mathbb{E}_{\widetilde P}f_\varepsilon-\mathbb{E}_{P_1}f_\varepsilon\big| \\
& \le 2\max\Big\{\big|\mathbb{E}_{P_0}f_\varepsilon-\mathbb{E}_{\widetilde P}f_\varepsilon\big|,\big|\mathbb{E}_{P_1}f_\varepsilon-\mathbb{E}_{\widetilde P}f_\varepsilon\big|\Big\},
\end{aligned}
\end{equation*}
which implies 
\begin{equation*}
\max\Big\{\big|\mathbb{E}_{P_0}f_\varepsilon-\mathbb{E}_{\widetilde P}f_\varepsilon\big|,\big|\mathbb{E}_{P_1}f_\varepsilon-\mathbb{E}_{\widetilde P}f_\varepsilon\big|\Big\}
\ge \mathrm{TV}(P_0,P_1)-\varepsilon.
\end{equation*}
Since $\varepsilon>0$ is arbitrary, taking the supremum over $\|f\|_\infty\le 1$ yields
\begin{equation*}
\sup_{\|f\|_\infty\le 1}\max\Big\{\big|\mathbb{E}_{P_0}f-\mathbb{E}_{\widetilde P}f\big|,\big|\mathbb{E}_{P_1}f-\mathbb{E}_{\widetilde P}f\big|\Big\}
\ge \mathrm{TV}(P_0,P_1).
\end{equation*}
Finally, using the variational characterization (for any distributions $R,S$),

\begin{equation*}
\sup_{\|f\|_\infty\le 1}|\mathbb{E}_{R}f-\mathbb{E}_{S}f|=2\,\mathrm{TV}(R,S)
\end{equation*}

and applying it with $(R,S)=(P_0,\widetilde P)$ and $(P_1,\widetilde P)$ gives
\begin{equation*}
    \begin{aligned}
\max\Big\{\sup_{\|f\|_\infty\le 1}\big|\mathbb{E}_{P_0}f-\mathbb{E}_{\widetilde P}f\big|,
\sup_{\|f\|_\infty\le 1}\big|\mathbb{E}_{P_1}f-\mathbb{E}_{\widetilde P}f\big|\Big\} \\
=2\max\{\mathrm{TV}(P_0,\widetilde P),\mathrm{TV}(P_1,\widetilde P)\}.
\end{aligned}
\end{equation*}
Since $\max\{\sup a_f,\sup b_f\}\ge \sup \max\{a_f,b_f\}$, combining with the previous inequality yields
\begin{equation*}
    \max\{\mathrm{TV}(P_0,\widetilde P),\mathrm{TV}(P_1,\widetilde P)\}\ge \tfrac12\,\mathrm{TV}(P_0,P_1)
\end{equation*}

\end{proof}

\subsection{Posterior-Predictive Approximation via Decoupled Heads}\label{pf:postpred-decoupled}

The Bayes posterior predictive integrates over the posterior of missing features, conditional on the observed features and the available context data.
In our implementation, the PFN head is trained to \emph{directly} output the marginalized posterior predictive $P_\phi(\cdot\mid x_m^c,D_m^c)$ from incomplete contexts (implicit marginalization), while the flow head optionally learns an explicit missing-value posterior $Q_\theta(x_m\mid x_m^c,D_m^c)$ for sampling/imputation and uncertainty analysis.
The lemmas below provide a proof chain that (i) characterizes posterior integration as a marginalization identity, (ii) decomposes decoupled approximation error into posterior-mismatch and conditional-prediction terms, (iii) controls posterior mismatch by a $W_1$ distance under Lipschitz regularity, and (iv) links PFN-style conditional KL training to bounded test-function error via total variation and Pinsker's inequality.

\begin{lemma}[Posterior-predictive marginalization]\label{lem:postpred-marginalization}
Let $D_m^c$ denote the context data available at test time, and let $x_m^c$ and $x_m$ be the observed and missing parts of $x$, respectively.
Assuming the conditional distributions are well-defined, the Bayes posterior predictive satisfies
\begin{equation*}
\begin{aligned}
    &P^\star(y\mid x_m^c, D_m^c)\\
    =& \int P^\star(y\mid x_m, x_m^c, D_m^c)\,P^\star(x_m\mid x_m^c, D_m^c)\,dx_m.
\end{aligned}
\end{equation*}
\begin{proof}
This is the law of total probability (or iterated expectation) applied to the latent missing features $x_m$:
\begin{equation*}
\begin{aligned}
    &P^\star(y\mid x_m^c, D_m^c)\\
    =& \int P^\star(y,x_m\mid x_m^c,D_m^c)\,dx_m\\
    =& \int P^\star(y\mid x_m,x_m^c,D_m^c)\,P^\star(x_m\mid x_m^c,D_m^c)\,dx_m.
\end{aligned}
\end{equation*}
\end{proof}
\end{lemma}

\begin{lemma}[Decoupled approximation error decomposition]\label{lem:postpred-decompose}
Fix $(x_m^c,D_m^c)$. Let $G$ be a posterior generator that induces a distribution $G(\cdot\mid x_m^c,D_m^c)$ over $x_m$, and let
$H(\cdot\mid x_m,x_m^c,D_m^c)$ be a conditional predictor (e.g., a classifier).
Define the induced posterior-predictive mixture
\begin{equation*}
\begin{aligned}
    &\widehat{P}(y\mid x_m^c,D_m^c)\\
    :=
    & \int H(y\mid x_m,x_m^c,D_m^c)\,G(x_m\mid x_m^c,D_m^c)\,dx_m.
\end{aligned}
\end{equation*}
Then, for any measurable function $\varphi$ with $\|\varphi\|_\infty\le 1$,
\begin{equation*}
\begin{aligned}
    & \big|\mathbb{E}_{\widehat{P}}[\varphi(y)]-\mathbb{E}_{P^\star}[\varphi(y)]\big|
    \le \\
    & \underbrace{\big|\mathbb{E}_{x_m\sim G}\mathbb{E}_{y\sim H}[\varphi(y)] 
    -\mathbb{E}_{x_m\sim P^\star(\cdot\mid x_m^c,D_m^c)}\mathbb{E}_{y\sim H}[\varphi(y)]\big|}_{\text{posterior approximation (learn $G$)}} 
    +\\
    & \underbrace{\mathbb{E}_{x_m\sim P^\star(\cdot\mid x_m^c,D_m^c)}
    \left|
    \begin{aligned}
        & \mathbb{E}_{y\sim H}[\varphi(y)] -
        \\
        & \mathbb{E}_{y\sim P^\star(\cdot\mid x_m,x_m^c,D_m^c)}[\varphi(y)]
    \end{aligned}
    \right|}_{\text{conditional prediction (learn $H$)}}.
\end{aligned}
\end{equation*}
\begin{proof}
By definition of $\widehat{P}$,
\begin{equation*}
\mathbb{E}_{\widehat{P}}[\varphi(y)]
=\mathbb{E}_{x_m\sim G(\cdot\mid x_m^c,D_m^c)}\ \mathbb{E}_{y\sim H(\cdot\mid x_m,x_m^c,D_m^c)}[\varphi(y)].
\end{equation*}
Also, by Lemma~\ref{lem:postpred-marginalization},
\begin{equation*}
\mathbb{E}_{P^\star}[\varphi(y)]
=\mathbb{E}_{x_m\sim P^\star(\cdot\mid x_m^c,D_m^c)}\ \mathbb{E}_{y\sim P^\star(\cdot\mid x_m,x_m^c,D_m^c)}[\varphi(y)].
\end{equation*}
Add and subtract $\mathbb{E}_{x_m\sim P^\star(\cdot\mid x_m^c,D_m^c)}\mathbb{E}_{y\sim H}[\varphi(y)]$ and apply the triangle inequality to obtain the stated bound.
\end{proof}
\end{lemma}

\begin{lemma}[Bounding posterior mismatch by $W_1$]\label{lem:postpred-w1}
Assume that for fixed $(x_m^c,D_m^c)$ and for any $\varphi$ with $\|\varphi\|_\infty\le 1$, the map
$x_m\mapsto \mathbb{E}_{y\sim H(\cdot\mid x_m,x_m^c,D_m^c)}[\varphi(y)]$ is $L_h$-Lipschitz w.r.t.\ $\|\cdot\|_2$.
Then the first term in Lemma~\ref{lem:postpred-decompose} is bounded by
\begin{equation*}
    L_h\,W_1\!\Big(G(\cdot\mid x_m^c,D_m^c),\,P^\star(\cdot\mid x_m^c,D_m^c)\Big),
\end{equation*}
where $W_1$ is the Wasserstein-1 distance.
\begin{proof}
Fix $(x_m^c,D_m^c)$ and define
\begin{equation*}
f(x_m):=\mathbb{E}_{y\sim H(\cdot\mid x_m,x_m^c,D_m^c)}[\varphi(y)].
\end{equation*}
By assumption, $f$ is $L_h$-Lipschitz. By Kantorovich--Rubinstein duality,
\begin{equation*}
\begin{aligned}
    & \big|\mathbb{E}_{x_m\sim G}f(x_m)-\mathbb{E}_{x_m\sim P^\star(\cdot\mid x_m^c,D_m^c)}f(x_m)\big| \\
    & \le L_h\,W_1\!\Big(G(\cdot\mid x_m^c,D_m^c),\,P^\star(\cdot\mid x_m^c,D_m^c)\Big),
\end{aligned}
\end{equation*}
which is exactly the desired bound.
\end{proof}
\end{lemma}

\begin{lemma}[From conditional KL to bounded test functions (Pinsker)]\label{lem:postpred-pinsker}
Fix $(x_m^c,D_m^c)$. Let $R:=P^\star(\cdot\mid x_m^c,D_m^c)$ and let $S$ be any learned posterior predictive distribution over $y$ (e.g., a PFN output $S:=P_\phi(\cdot\mid x_m^c,D_m^c)$).
Define total variation as $\mathrm{TV}(R,S):=\sup_{A}|R(A)-S(A)|$.
Then for any measurable $\varphi$ with $\|\varphi\|_\infty\le 1$,
\[
\big|\mathbb{E}_{R}[\varphi]-\mathbb{E}_{S}[\varphi]\big|\le 2\,\mathrm{TV}(R,S)
\le \sqrt{2\,\mathrm{KL}(R\|S)}.
\]
\end{lemma}
\begin{proof}
The first inequality follows from the variational characterization
$\sup_{\|\varphi\|_\infty\le 1}\big|\mathbb{E}_{R}\varphi-\mathbb{E}_{S}\varphi\big|=2\,\mathrm{TV}(R,S)$ under the above convention for $\mathrm{TV}$.
The second inequality is Pinsker's inequality, $\mathrm{TV}(R,S)\le \sqrt{\tfrac12\,\mathrm{KL}(R\|S)}$.
Combining the two bounds yields the claim.
\end{proof}

\begin{proposition}[Why decoupling helps approximate Bayes posterior predictive with limited context]\label{prop:decoupled-bayes-proxy}
Under the identifiability condition in Remark~\ref{rem:mnar-ident} (so that $P^\star(x_m\mid x_m^c,D_m^c)$ is well-defined within the model class), Lemmas~\ref{lem:postpred-decompose}--\ref{lem:postpred-w1} show that approximating the Bayes posterior predictive reduces to separately learning (i) the posterior over missing features ($G$) and (ii) the conditional predictor ($H$).
Moreover, Lemma~\ref{lem:postpred-pinsker} shows how PFN-style training that minimizes conditional cross-entropy (equivalently, conditional KL in the realizable case) controls bounded test-function error for the \emph{marginalized} posterior predictive without explicitly computing the integral in Lemma~\ref{lem:postpred-marginalization}.
Combined with Theorem~\ref{prop:decoupled-sample}, this supports the claim that, with limited context (test-time training / in-context learning), a decoupled multi-task design can be more sample-efficient and thereby more favorable for approximating the Bayes-optimal posterior predictive than directly learning the coupled end-to-end map.
\begin{proof}
Under Remark~\ref{rem:mnar-ident}, the Bayes posterior predictive is well-defined in-model.
Lemmas~\ref{lem:postpred-decompose}--\ref{lem:postpred-w1} decompose the approximation error into (i) a posterior-mismatch term (bounded by a $W_1$ distance for $G$) and (ii) a conditional-prediction term for $H$.
Hence, approximating the posterior predictive reduces to learning $G$ and $H$ to control these two terms.
Lemma~\ref{lem:postpred-pinsker} provides a direct route for the PFN head: small conditional KL to the true posterior predictive implies small bounded test-function error for the marginalized predictor.
Theorem~\ref{prop:decoupled-sample} adds a complementary complexity argument: when the structural component is highly nonlinear (large effective $L_g$) and context is limited (large coupling difficulty), learning the decoupled components can be more sample-efficient than learning the coupled end-to-end map in the ICL regime.
\end{proof}
\end{proposition}

\begin{corollary}[Explicit link to complexity via a two-term target]\label{cor:postpred-two-term}
Under the assumptions of Lemmas~\ref{lem:postpred-decompose}--\ref{lem:postpred-w1}, for any $\varphi$ with $\|\varphi\|_\infty\le 1$,
\begin{equation*}
\begin{aligned}
    & \big|\mathbb{E}_{\widehat{P}}[\varphi(y)]-\mathbb{E}_{P^\star}[\varphi(y)]\big|
    \le \\
    & L_h\,W_1\!\Big(G(\cdot\mid x_m^c,D_m^c),\,P^\star(\cdot\mid x_m^c,D_m^c)\Big)
    + \epsilon_H(x_m^c,D_m^c), \\
\end{aligned}
\end{equation*}
where $\epsilon_H(x_m^c,D_m^c):=\mathbb{E}_{x_m\sim P^\star(\cdot\mid x_m^c,D_m^c)}
\big|\mathbb{E}_{y\sim H}[\varphi(y)]-\mathbb{E}_{y\sim P^\star}[\varphi(y)]\big|$ is the conditional prediction error term.
In particular, to ensure the total error is at most $\epsilon$, it suffices to budget
$W_1(G,P^\star)\le \epsilon/(2L_h)$ and $\epsilon_H(x_m^c,D_m^c)\le \epsilon/2$, making the required error propagation through $L_h$ explicit.
Thus, approximating the posterior predictive can be viewed as a two-term learning target (learn $G$ to reduce $W_1$, learn $H$ to reduce $\epsilon_H$), aligning with Theorem~\ref{prop:decoupled-sample} where the decoupled sample-complexity proxy depends on the effective complexities of $G$ and $H$ and the ICL coupling difficulty of learning the coupled map directly.
\begin{proof}
Apply Lemma~\ref{lem:postpred-decompose} and then upper bound the first (posterior approximation) term by Lemma~\ref{lem:postpred-w1}. The remaining term is exactly $\epsilon_H(x_m^c,D_m^c)$ by definition.
\end{proof}
\end{corollary}

\subsection{Proof of Theorem~\ref{prop:decoupled-sample}}\label{pf:decoupled-sample}
\begin{proof}
We make explicit the proxy assumptions used by Theorem~\ref{prop:decoupled-sample}.
Assume there exists a function $\mathcal{N}(L,\epsilon)$ such that, for a $d$-dimensional learning target with effective complexity parameter $L$, achieving error at most $\epsilon$ requires
\[
\mathcal{N}(L,\epsilon)=\mathcal{O}\!\big((L/\epsilon)^d\big).
\]
Define $L_g$ as the effective complexity parameter governing the difficulty of learning the missing-feature posterior map
$(x_m^c,D_m^c)\mapsto P^\star(x_m\mid x_m^c,D_m^c)$ to Wasserstein-1 accuracy, and define $L_{\mathrm{cpl}}\ge 0$ as an additional coupling penalty for learning the coupled end-to-end posterior-predictive map directly from limited context.

By Corollary~\ref{cor:postpred-two-term}, to ensure total test-function error at most $\epsilon$ it suffices to satisfy
\begin{equation*}
\begin{aligned}
    & W_1\!\Big(G(\cdot\mid x_m^c,D_m^c),\,P^\star(\cdot\mid x_m^c,D_m^c)\Big)\le \frac{\epsilon}{2L_h} \\
    & \text{and} \\
    & \epsilon_H(x_m^c,D_m^c)\le \frac{\epsilon}{2}.
\end{aligned}
\end{equation*}
Under the proxy sample-complexity law, learning $G$ to Wasserstein error $\epsilon/(2L_h)$ requires at most
\[
\mathcal{N}\!\left(L_g,\frac{\epsilon}{2L_h}\right)
=\mathcal{O}\!\Big(\big((L_gL_h)/\epsilon\big)^d\Big),
\]
and learning $H$ to conditional-prediction error $\epsilon/2$ requires at most
\[
\mathcal{N}\!\left(L_h,\frac{\epsilon}{2}\right)
=\mathcal{O}\!\Big(\big(L_h/\epsilon\big)^d\Big),
\]
where constant factors are absorbed into $\mathcal{O}(\cdot)$. Summing the two requirements yields
\[
\mathcal{N}_{Decoupled}
\approx
\mathcal{O}\Big(\big((L_gL_h)/\epsilon\big)^d\Big)+\mathcal{O}\Big(\big(L_h/\epsilon\big)^d\Big).
\]

For the coupled end-to-end predictor, apply the same proxy law to the $d$-dimensional target map $F(x_m^c,D_m^c)=P^\star(\cdot\mid x_m^c,D_m^c)$, and assume its effective complexity is upper bounded (up to constants) by
$L_h(1+L_g)+L_{\mathrm{cpl}}$.
Then achieving error at most $\epsilon$ requires
\begin{equation*}
\begin{aligned}
    & \mathcal{N}_{E2E}
\approx
    \mathcal{N}\!\big(L_h(1+L_g)+L_{\mathrm{cpl}},\epsilon\big) \\
    & =
    \mathcal{O}\Big(\big((L_h(1+L_g)+L_{\mathrm{cpl}})/\epsilon\big)^d\Big),
\end{aligned}
\end{equation*}
which is exactly the stated comparison.
\end{proof}


\begin{figure*}[!t]
  \centering
  \includegraphics[width=\textwidth]{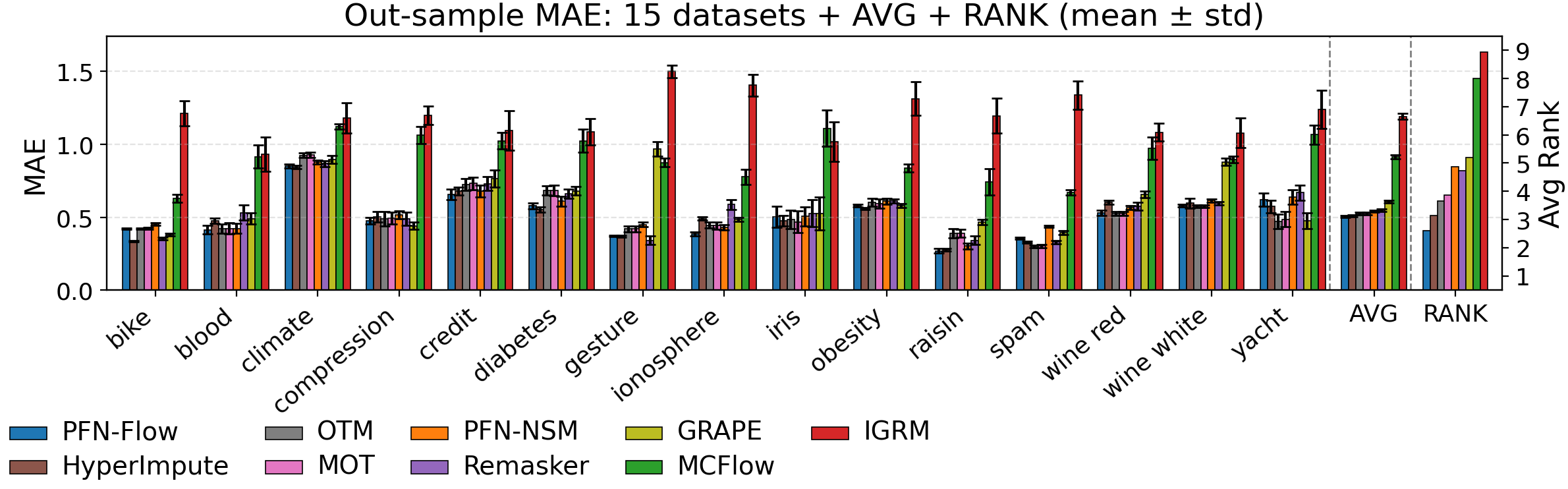}
  \caption{Imputation benchmark under MCAR missingness: out-of-sample MAE over 15 datasets, plus AVG and RANK (mean $\pm$ std over masks).}
  \label{fig:outsample-mae-mcar}
\end{figure*}

\section{Additional Experiments}\label{app:additional-experiments}
We report additional imputation experiments under the MCAR missingness mechanism.

Figure~\ref{fig:outsample-mae-mcar} summarizes out-of-sample MAE across the same 15 imputation benchmark datasets, together with the overall average (AVG) and the average rank (RANK).

We additionally compare (i) our method against boosting baselines, (ii) TabPFN/PFN-family variants, and (iii) top split baselines, in that order.

\subsection{Regression results under MCAR missingness}
Across the 15 datasets in Figure~\ref{fig:outsample-mae-mcar}, PFN-Flow is consistently competitive under MCAR missingness.
In particular, PFN-Flow tends to achieve either the best or near-best MAE on a majority of datasets, and remains in the top tier in terms of AVG and RANK, indicating both strong average performance and stable behavior across datasets and masks.
We also observe that some baselines can be highly dataset-dependent (large performance variance across datasets), whereas PFN-Flow exhibits comparatively uniform improvements.

\subsection{Boosting baselines}
Figure~\ref{fig:additional-vs-methods} compares PFN-Flow with three standard boosting baselines (XGBoost/CatBoost/LightGBM, trained on the same MCAR-imputed data).
Each scatter contains two panels corresponding to missing rate $\le 10\%$ and $>10\%$; for each dataset, the x-axis is PFN-Flow AUC and the y-axis is the baseline AUC.
Points below the diagonal therefore indicate that PFN-Flow outperforms the baseline on that dataset/group, while points above the diagonal indicate the opposite.
Overall, most points lie on or below the diagonal, suggesting that PFN-Flow remains competitive against strong non-neural baselines under MCAR, with only a small number of datasets where boosting can be slightly better.

\subsection{TabPFN variants and PFN-family baselines}
Figure~\ref{fig:additional-pfn-family} further compares PFN-Flow to TabPFN (Raw) and a PFN-family baseline (PFN-NSM).
We observe that the points are highly concentrated near the diagonal in both missingness regimes, indicating that PFN-Flow is broadly comparable to TabPFN/PFN-family methods on the same datasets under MCAR.
Notably, in the higher-missingness group ($>10\%$), the scatter remains close to parity with only a few outliers, suggesting that PFN-Flow does not rely on a narrow set of favorable datasets and maintains stable performance as missingness increases.

\subsection{Top split baselines}
Figure~\ref{fig:additional-top4-split} reports the strongest split baselines selected from our splitting-based design space (Top-1 to Top-4), each corresponding to a \emph{pipeline} of a predictive model (here, CatBoost) and a specific imputation module.
Within each subfigure, the two internal panels again correspond to missing rate $\le 10\%$ and $>10\%$.
For each dataset, the x-axis is PFN-Flow AUC and the y-axis is the baseline pipeline AUC; points below the diagonal indicate that PFN-Flow improves upon that strong split baseline.
Across all four top pipelines, the majority of points remain on or below the diagonal, showing that PFN-Flow is competitive even against the best-performing split configurations.
At the same time, the few points above the diagonal highlight that no single method universally dominates every dataset, motivating the need for robust imputation that generalizes across heterogeneous tabular domains.

\begin{figure*}[t]
  \centering
  \begin{subfigure}[t]{0.48\linewidth}
    \centering
    \includegraphics[width=\linewidth]{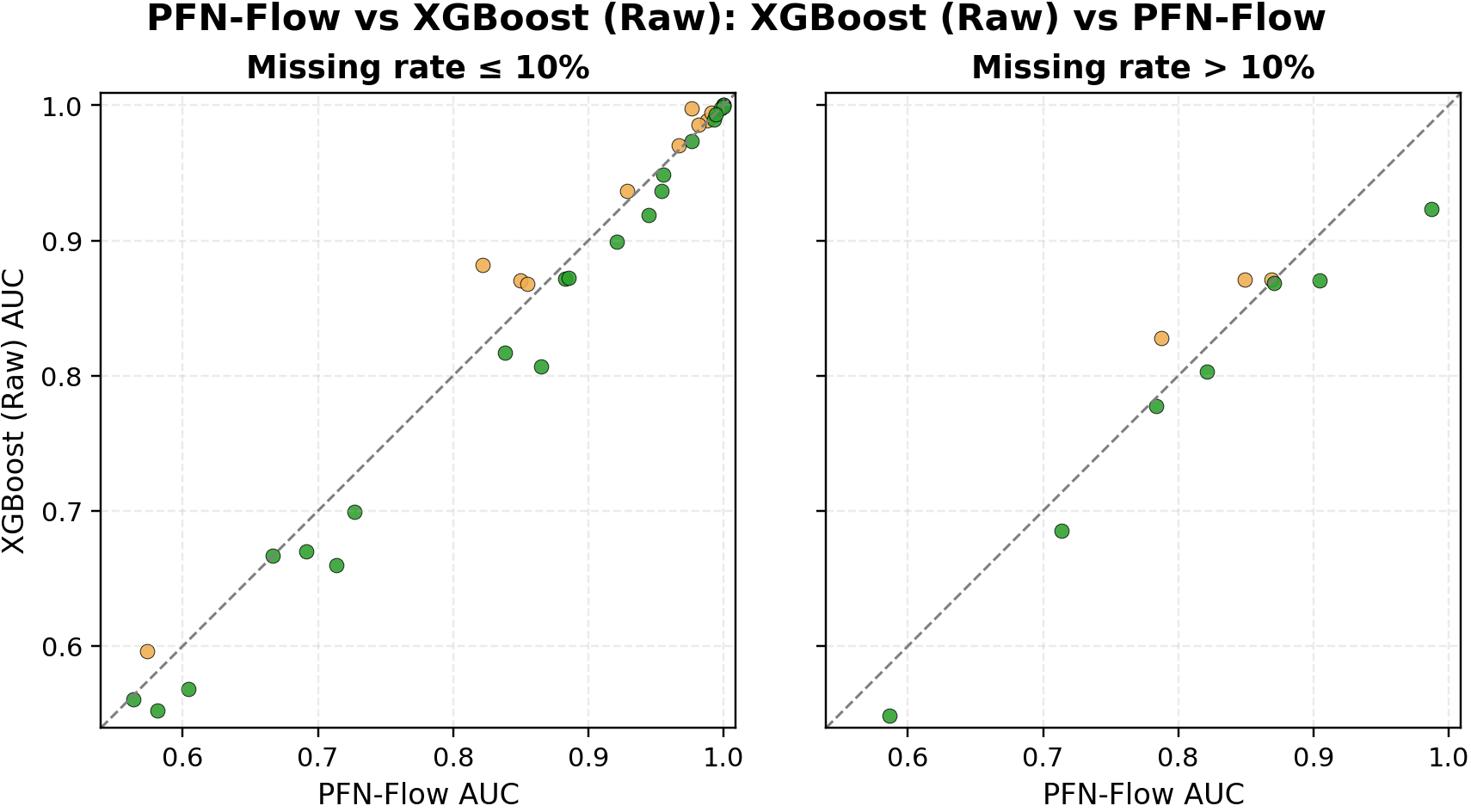}
    \caption{XGBoost (Raw) vs PFN-Flow (AUC).}
  \end{subfigure}
  \begin{subfigure}[t]{0.48\linewidth}
    \centering
    \includegraphics[width=\linewidth]{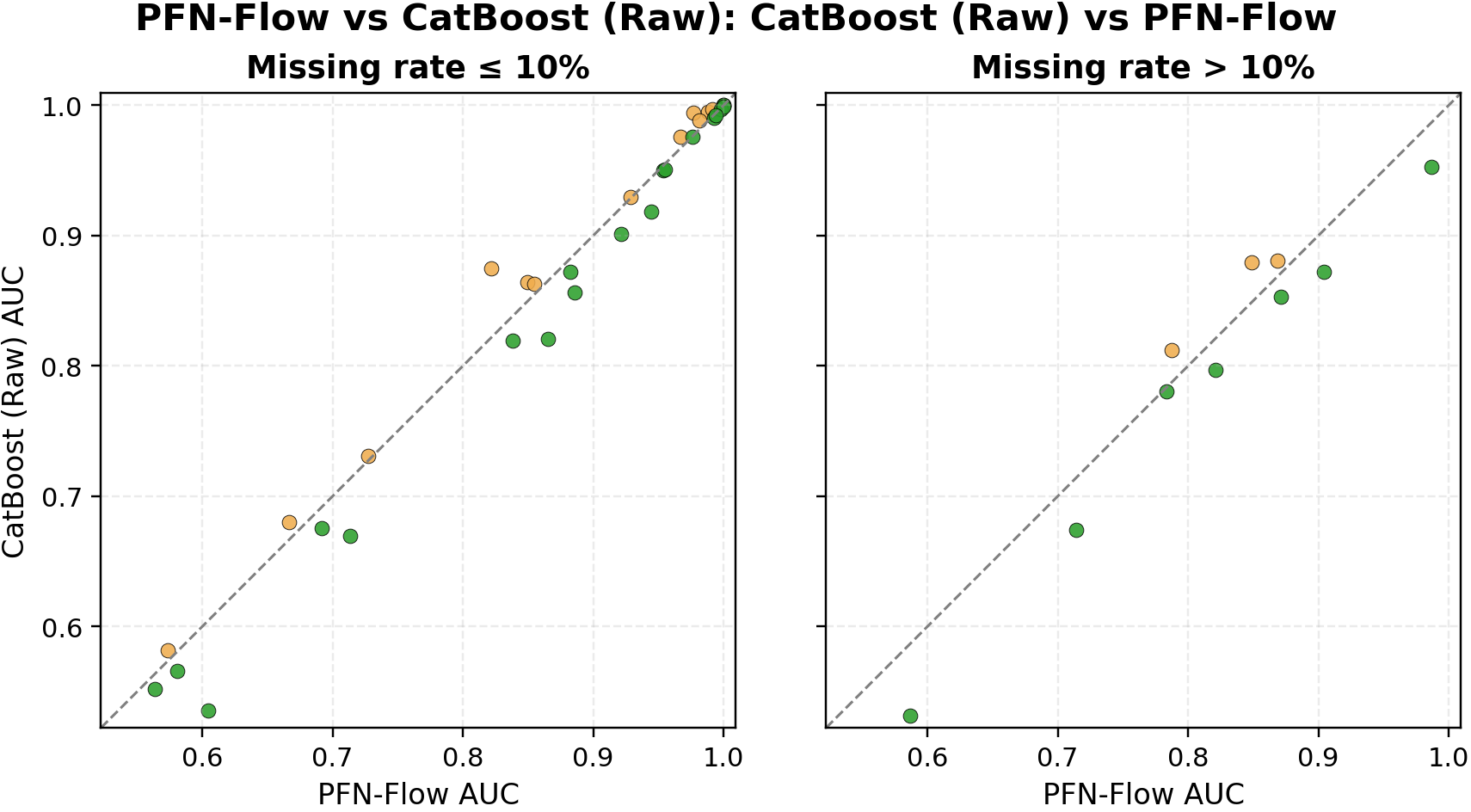}
    \caption{CatBoost (Raw) vs PFN-Flow (AUC).}
  \end{subfigure}

  \begin{subfigure}[t]{0.48\linewidth}
    \centering
    \includegraphics[width=\linewidth]{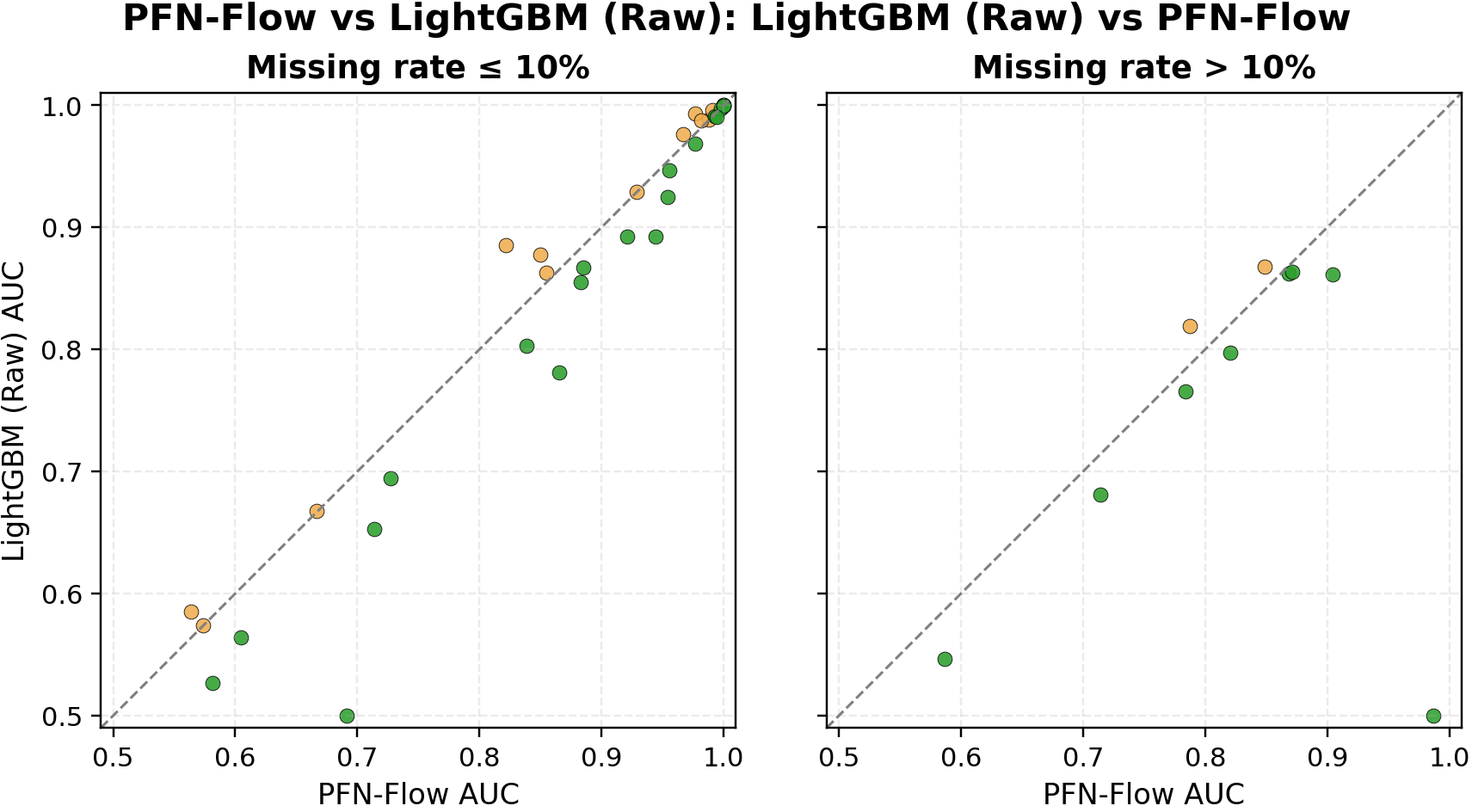}
    \caption{LightGBM (Raw) vs PFN-Flow (AUC).}
  \end{subfigure}
  \caption{Boosting baseline comparisons against PFN-Flow under MCAR. Each point corresponds to one dataset in one missingness group; the diagonal indicates parity.}
  \label{fig:additional-vs-methods}
\end{figure*}
\FloatBarrier

\begin{figure*}[t]
  \centering
  \begin{subfigure}[t]{0.48\linewidth}
    \centering
    \includegraphics[width=\linewidth]{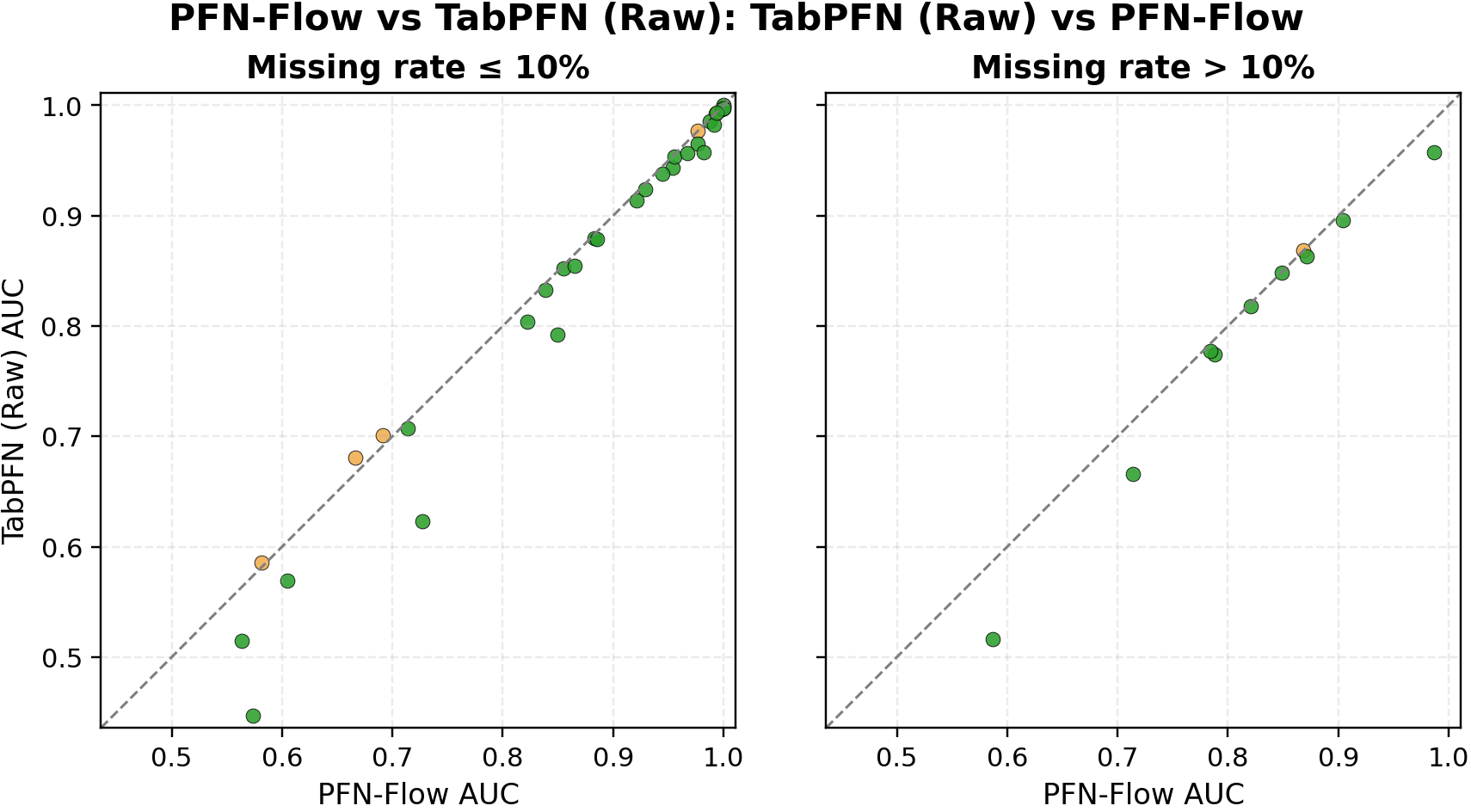}
    \caption{TabPFN (Raw) vs PFN-Flow (AUC).}
  \end{subfigure}
  \begin{subfigure}[t]{0.48\linewidth}
    \centering
    \includegraphics[width=\linewidth]{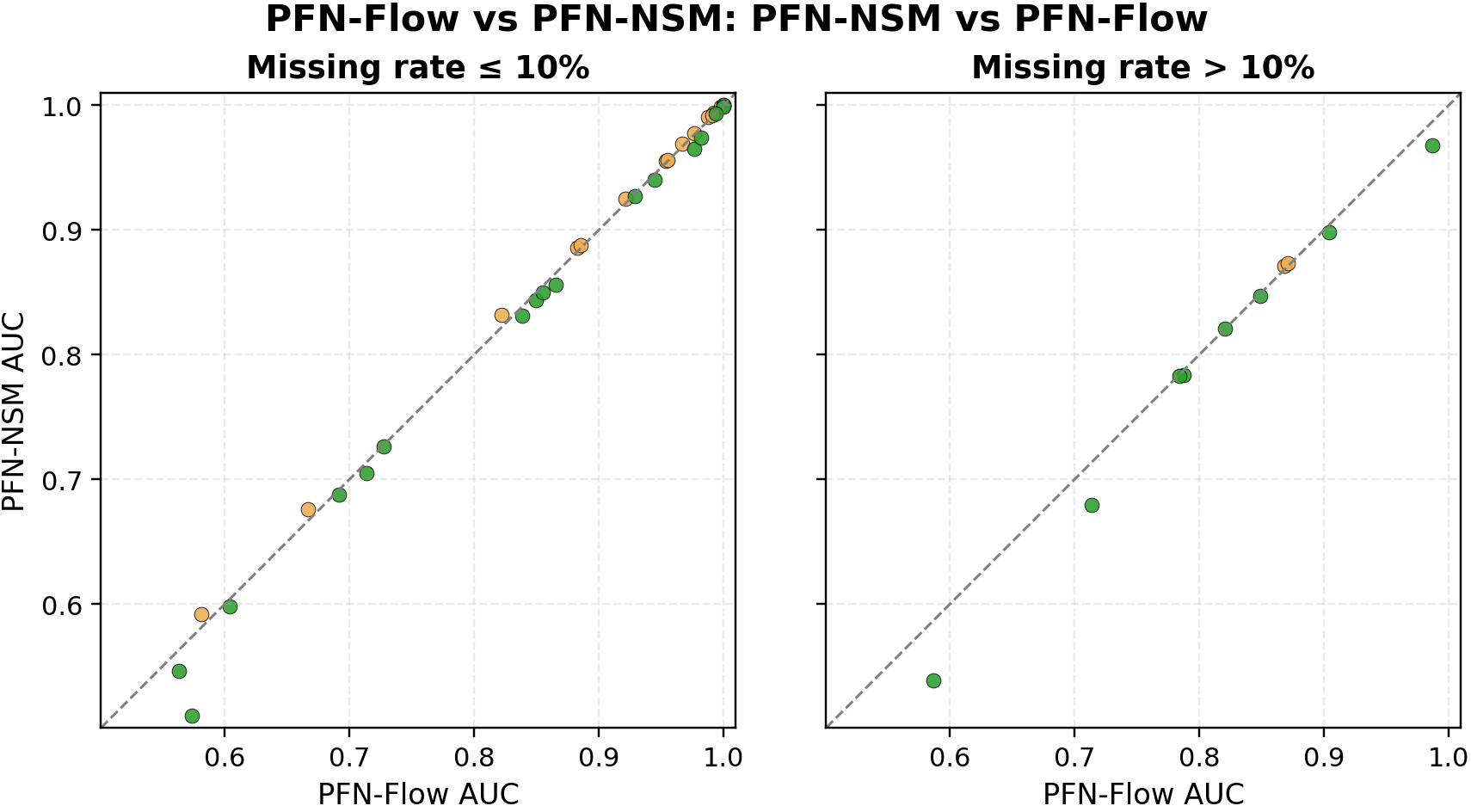}
    \caption{PFN-NSM vs PFN-Flow (AUC).}
  \end{subfigure}
  \caption{TabPFN/PFN-family comparisons against PFN-Flow under MCAR. Each point corresponds to one dataset in one missingness group; the diagonal indicates parity.}
  \label{fig:additional-pfn-family}
\end{figure*}

\begin{figure*}[t]
  \centering
  \begin{subfigure}[t]{0.54\linewidth}
    \centering
    \includegraphics[width=\linewidth]{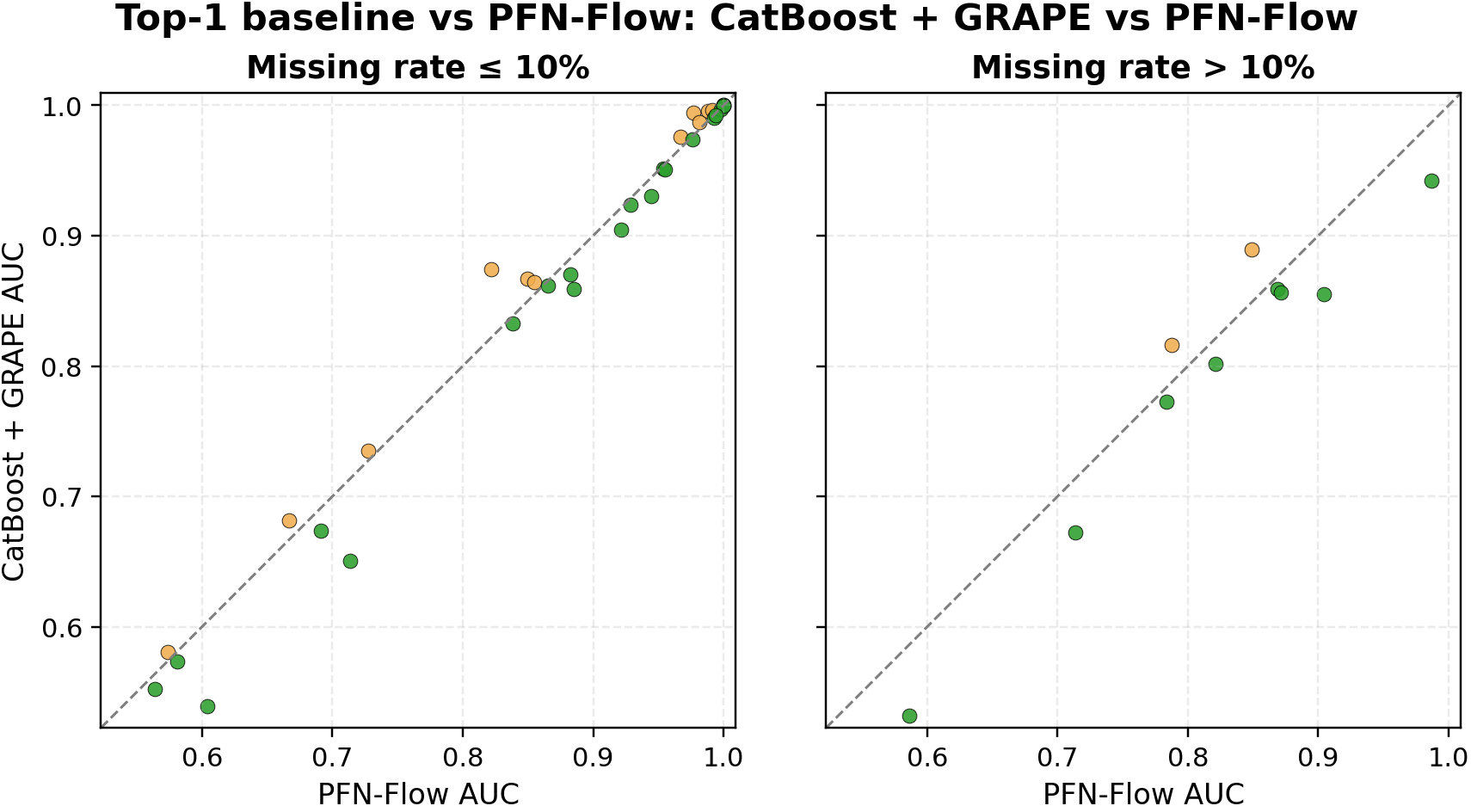}
    \caption{Top-1 split: CatBoost + GRAPE.}
  \end{subfigure}

  \begin{subfigure}[t]{0.54\linewidth}
    \centering
    \includegraphics[width=\linewidth]{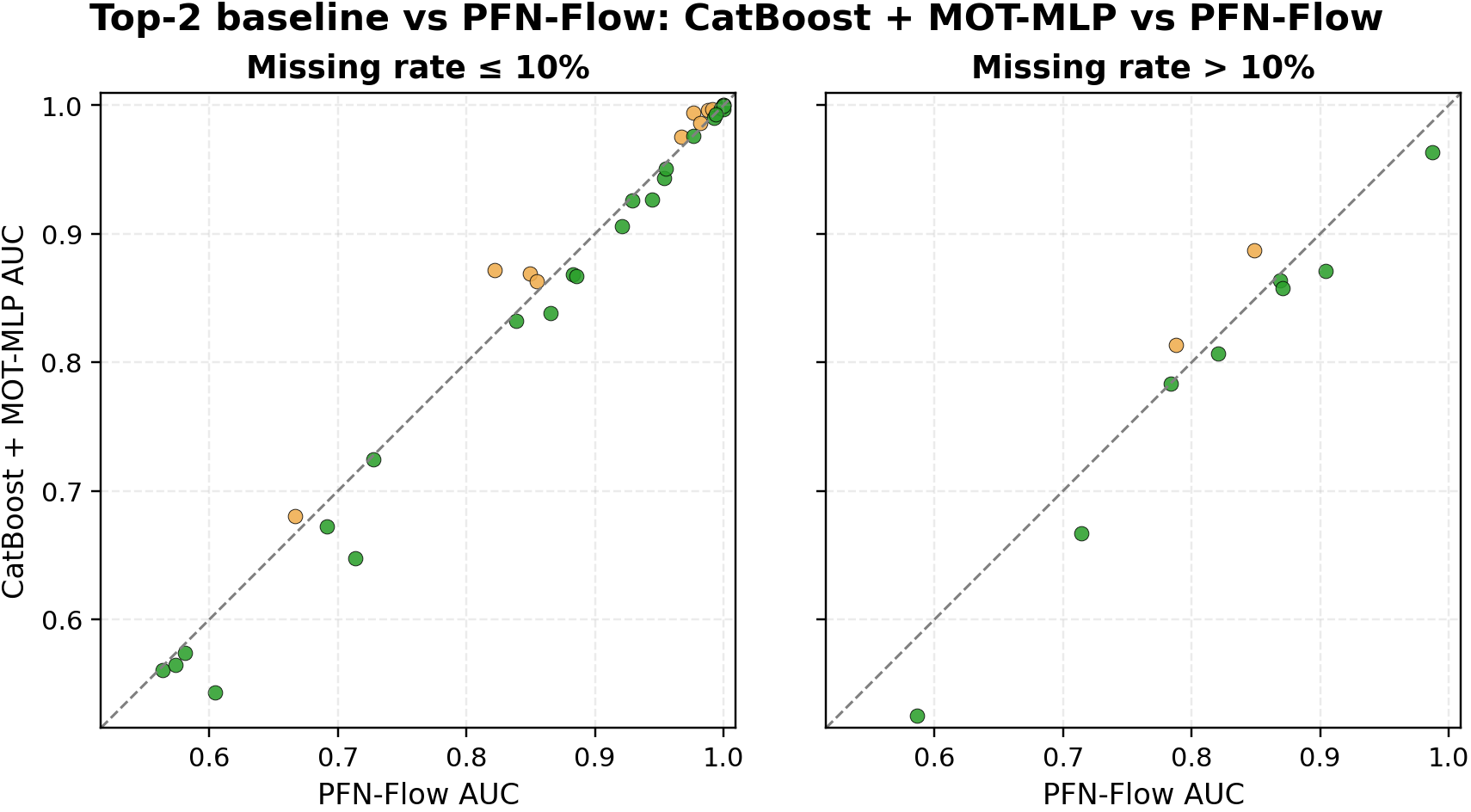}
    \caption{Top-2 split: CatBoost + MOT-MLP.}
  \end{subfigure}

  \begin{subfigure}[t]{0.54\linewidth}
    \centering
    \includegraphics[width=\linewidth]{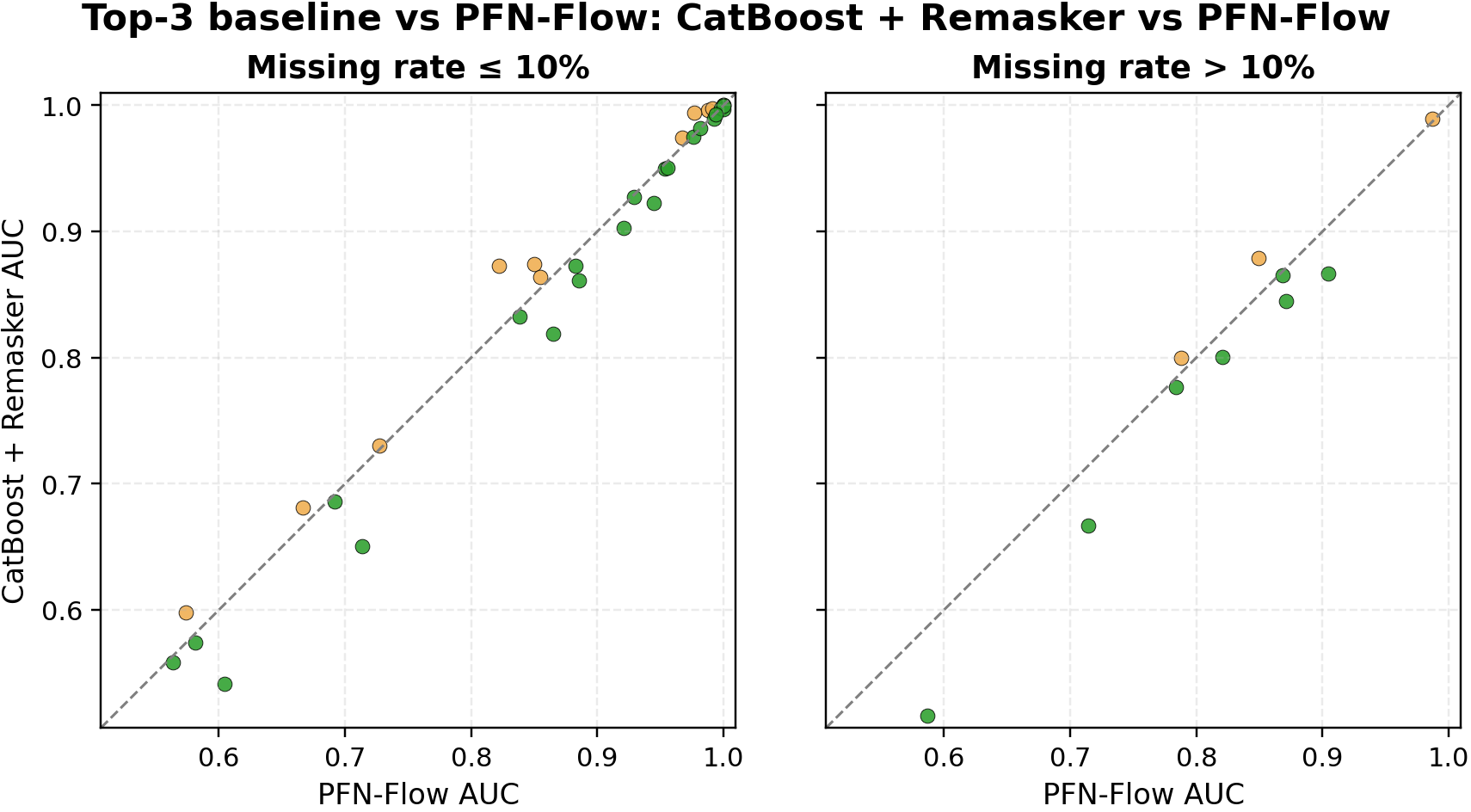}
    \caption{Top-3 split: CatBoost + Remasker.}
  \end{subfigure}

  \begin{subfigure}[t]{0.54\linewidth}
    \centering
    \includegraphics[width=\linewidth]{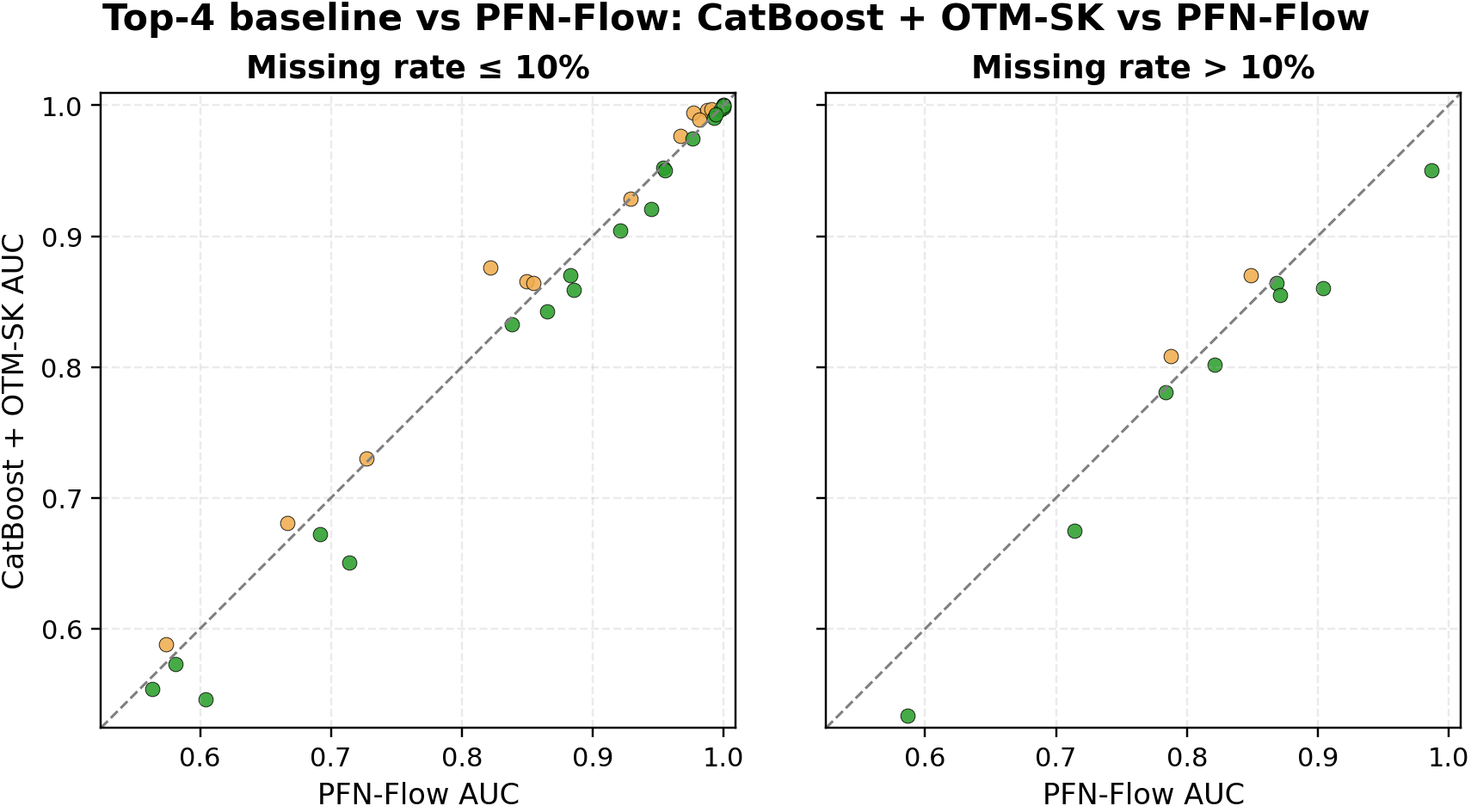}
    \caption{Top-4 split: CatBoost + OTM-SK.}
  \end{subfigure}

  \caption{Top split comparisons.}
  \label{fig:additional-top4-split}
\end{figure*}

\FloatBarrier

\end{document}